\documentclass[final]{l4dc2026}

\usepackage{url}            
\usepackage{booktabs}       
\usepackage{amsfonts}       
\usepackage{nicefrac}       
\usepackage{microtype}      

\usepackage{algorithm}  

\usepackage{amssymb}		
\usepackage{graphicx}		
\usepackage{amsmath}
\usepackage{algorithmic}
\usepackage{comment}
\usepackage{multirow}
\usepackage{enumitem}
\usepackage{subcaption}
\usepackage{tikz}
\usetikzlibrary{arrows.meta}
\usepackage[dvipsnames,table,xcdraw]{xcolor}

\newtheorem*{assumption*}{Assumption}

\DeclareMathOperator*{\E}{\mathbb{E}}

\usepackage{amsmath}  
\usepackage{xspace}   

\usepackage{pifont}
\usepackage{float}

\usepackage{amsmath,amssymb,mathtools}

\newcommand{\KL}{\mathrm{KL}}         

\usepackage[nameinlink]{cleveref}  
\crefname{theorem}{Theorem}{Theorems}
\Crefname{theorem}{Theorem}{Theorems}

\crefname{lemma}{Lemma}{Lemmas}
\Crefname{lemma}{Lemma}{Lemmas}

\crefname{proposition}{Proposition}{Propositions}
\Crefname{proposition}{Proposition}{Propositions}

\crefname{assumption}{Assumption}{Assumptions}
\Crefname{assumption}{Assumption}{Assumptions}

\crefname{algorithm}{Algorithm}{Algorithms}
\Crefname{algorithm}{Algorithm}{Algorithms}

\crefname{figure}{Figure}{Figures}
\Crefname{figure}{Figure}{Figures}

\crefname{corollary}{Corollary}{Corollaries}
\Crefname{corollary}{Corollary}{Corollaries}

\crefname{remark}{Remark}{Remarks}
\Crefname{remark}{Remark}{Remarks}

\crefname{definition}{Definition}{Definitions}
\Crefname{definition}{Definition}{Definitions}

\crefname{section}{Section}{Sections}
\Crefname{section}{Section}{Sections}

\crefname{subsection}{Section}{Sections}
\Crefname{subsection}{Section}{Sections}

\newcommand{\TV}{\mathrm{TV}}
\newcommand{\R}{\mathbb{R}}
\newcommand{\Sset}{\mathcal{S}}
\newcommand{\Aset}{\mathcal{A}}

\title[Locality in MARL]{A Unified Framework for Locality in Scalable MARL}
\usepackage{times}

\newcommand{\CU}{\textsuperscript{\textdagger}}      
\newcommand{\INRIA}{\textsuperscript{\textdaggerdbl}}  

\author{%
  \Name{Sourav Chakraborty}\CU \Email{sourav.chakraborty@colorado.edu}\\
  \Name{Amit Kiran Rege}\CU\footnote{Equal Contribution with S. Chakraborty.} \Email{amit.rege@colorado.edu}\\
  \Name{Claire Monteleoni}\CU\INRIA \Email{claire.monteleoni@colorado.edu}\\
  \Name{Lijun Chen}\CU \Email{lijun.chen@colorado.edu}\\
  \addr\CU University of Colorado, Boulder, USA.\\
  \addr\INRIA INRIA Paris, France.%
}

\begin{document}
\maketitle

\begin{abstract}
Scalable methods for networked multi-agent reinforcement learning let each agent plan using only a small neighborhood of the agent graph. This works only when the system is value-local, meaning a perturbation at one agent affects the long-run value at another agent weakly when the two are far apart. In the average-reward setting, the standard way to certify locality is the Dobrushin row-sum bound on a single matrix $C^\pi$ that captures how each agent's next state depends on each other agent's current state. To make this matrix easy to work with, prior work bounds it by a supremum over joint actions. The resulting bound is independent of the policy, but it is loose whenever the policy never picks the worst-case action. We split $C^\pi$ into pieces that separately track environment sensitivity and policy sensitivity, $C^\pi \preceq E^{\mathrm s}+E^{\mathrm a}\Pi(\pi)$, where $E^{\mathrm s}$ measures how the next state moves with the current state, $E^{\mathrm a}$ measures how it moves with the current action, and $\Pi(\pi)$ measures how reactive the policy is to changes in state. The spectral radius of $H^\pi := E^{\mathrm s}+E^{\mathrm a}\Pi(\pi)$ then controls the decay of the average-reward Poisson solution, and the spectral certificate $\rho(H^\pi)<1$ is strictly weaker than the row-sum condition $\|H^\pi\|_\infty<1$ on the same matrix and applies in regimes where policy-independent action-supremum bounds used in prior Dobrushin-style work cannot. For temperature-$\tau$ softmax policies we get $\Pi(\pi)\le L/(2\tau)$, so the softmax temperature directly controls locality. We use this decay result to give a deterministic oracle guarantee for a block-coordinate KL-proximal policy-improvement template whose truncation bias decays exponentially in the message-passing radius $\kappa$.
\end{abstract}

\begin{keywords}%
  Multi-agent Reinforcement Learning
\end{keywords}

\section{Introduction}

Cooperative multi-agent reinforcement learning on networked systems faces a curse of dimensionality: the joint state and action spaces grow exponentially in the number of agents $n$, so even when the model factors into local interactions, centralized planning is infeasible~\citep{complexity_blondel2000survey, complexity_papadimitriou1999complexity}. Scalable networked MARL methods address this by letting each agent plan using only a $\kappa$-hop neighborhood of the agent graph~\citep{qu_scalabale_marl_2020, qu2019scalable, lin2020distributed}. Their complexity scales with neighborhood size rather than network size. They are sound, though, only when the system itself is local, meaning that a perturbation at one agent affects the long-run value at a distant agent by an amount that decays exponentially in their graph distance. We call this property \emph{value-locality}. When it holds, $\kappa$-hop truncation costs an error that vanishes exponentially in $\kappa$.

So the question is when value-locality holds. In the standard $\gamma$-discounted setting it is automatic, because the discount factor $\gamma<1$ multiplies every step of the Bellman backup and produces decay regardless of the agent interaction structure. The average-reward setting ($\gamma=1$) has no such temporal multiplier. The role of the value function there is played by the \emph{bias function} $h^\pi$, which measures, for each starting state $s$, the long-run advantage of starting in $s$ relative to starting from the stationary distribution of $\pi$. It is defined up to an additive constant by the average-reward Poisson equation $h^\pi-T^\pi h^\pi=r^\pi-\bar r^\pi$, where $T^\pi$ is the one-step expectation operator under $\pi$ and $\bar r^\pi$ is the average reward. Whatever decay $h^\pi$ has must come from the decay of one-step influence under $T^\pi$. Existing average-reward guarantees~\citep{qu_scalabale_marl_2020} get this via a Dobrushin coupling condition. The idea is to build an $n\times n$ matrix $C^\pi$ whose entry $C^\pi_{j\leftarrow i}$ is the largest change a perturbation of $s_i$ can cause in the next-state distribution of agent $j$. If every row sum of $C^\pi$ is below one, then $T^\pi$ is a contraction in the seminorm $\|\delta(f)\|_\infty=\max_i \sup_{x_{-i}=y_{-i}}|f(x)-f(y)|$, the largest single-coordinate oscillation of $f$, and iterating the contraction gives exponential decay of $h^\pi$ in graph distance.

The Dobrushin condition is convenient because the bound on $C^\pi$ does not depend on the policy: one takes a supremum over joint actions when measuring how much an action change moves the next state. That is also why the condition is conservative. Consider two agents in which agent~$2$'s next state copies agent~$1$'s last action, regardless of state. The worst-case action move is maximal, so the supremum says the system is globally coupled. But if agent~$1$'s policy barely reacts to its own state, then flipping $s_1$ only weakly moves the distribution of $a_1$, which only weakly moves the next-state marginal of agent~$2$. In closed loop, $s_1$ and $s'_2$ are nearly independent, so the system is value-local. The Dobrushin test cannot see this because it threw the policy away.

To recover that information, we decompose the policy-induced one-step interdependence matrix into an environment piece and a policy piece. Using the total variation distance $\TV(\mu,\nu)=\tfrac12\sum_x|\mu(x)-\nu(x)|$, define $E^{\mathrm s}_{j\leftarrow i}$ as the worst-case TV between $P_j(\cdot\mid s,a)$ and $P_j(\cdot\mid s',a)$ over pairs $(s,s')$ that differ only on coordinate $i$, with the action fixed. Define $E^{\mathrm a}_{j\leftarrow i}$ the same way for an action change with the state fixed. Define $\Pi_{k\leftarrow i}(\pi)$ as the TV-sensitivity of agent $k$'s action distribution to a change in $s_i$. Of these, $E^{\mathrm s}$ and $E^{\mathrm a}$ depend only on the environment; $\Pi(\pi)$ depends on the policy. Our first result (Proposition~\ref{prop:decomposition}) is the entrywise bound
\[
C^\pi \ \preceq\ E^{\mathrm s} + E^{\mathrm a}\,\Pi(\pi).
\]
The product $E^{\mathrm a}\Pi(\pi)$ makes the cancellation visible. The action channel $E^{\mathrm a}$ can be large, but if the policy is smooth ($\Pi(\pi)$ small) the closed-loop influence is small anyway. Write $H^\pi := E^{\mathrm s}+E^{\mathrm a}\Pi(\pi)$. Our second result (Theorem~\ref{thm:poisson}) is that whenever the spectral radius $\rho(H^\pi)$ (the largest eigenvalue magnitude of $H^\pi$) is below one, the Poisson solution has $\delta(h^\pi)\le (I-(H^\pi)^\top)^{-1}\delta(r^\pi)$, where the right-hand side is the matrix-geometric (Neumann) series $\sum_{t\ge 0}((H^\pi)^\top)^t\delta(r^\pi)$. Our certificate is strictly weaker than prior ones in two senses. First, on the same comparison matrix, replacing the row-sum condition by the spectral-radius condition is strictly weaker since $\rho(M)\le\|M\|_\infty$ for any nonnegative $M$. Second, $H^\pi$ is policy-dependent and can be much smaller than the policy-independent action-supremum bounds used in prior Dobrushin-style guarantees, so the resulting certificate applies in regimes where those prior conditions are silent.

The policy factor $\Pi(\pi)$ is something a learning algorithm has control over. For temperature-$\tau$ softmax policies~\citep{Geist2019ATO, HaarnojaZAL18}, the policy class used in entropy-regularized control and KL-proximal updates, Lemma~\ref{lem:Pi-softmax} gives $\Pi_{k\leftarrow i}(\pi)\le \min\{1,\,L_{k\leftarrow i}/(2\tau)\}$, where $L_{k\leftarrow i}$ is the coordinatewise Lipschitz constant of the logit. Raising $\tau$ therefore directly tightens the certificate on $H^\pi$, and Section~\ref{sec:softmax} works out the tradeoff between certifying locality (higher $\tau$) and approaching the unregularized optimum (lower $\tau$).

Section~\ref{sec:algo} uses the decay result to give a deterministic oracle guarantee for a block-coordinate KL-proximal improvement template. The same Neumann tail $\lambda^{\kappa+1}$ appears in two places: in the locally computable certificate an agent uses to decide whether $\kappa$ hops are enough, and in the per-step improvement bound. The certificate is genuinely local. The truncated Poisson surrogate used in the analysis is an oracle object in the general model; turning it into something locally computable needs additional structure on observation scopes or function approximation, which we leave to follow-up work. We summarize related work next; the appendix has more.

\paragraph{Related work.}\label{sec:rel}
Exponential decay of value on networked MDPs was first studied in the scalable networked MARL line of work~\citep{qu_scalabale_marl_2020, qu2019scalable, lin_scalable_marl_stochastic_2021, lin2020distributed}, which proves discounted and average-reward locality results under a graph-local transition assumption $P_i(s'_i\mid s_{N_i},a_i)$ and a Dobrushin row-sum bound on $C^\pi$. Our setting allows $P_i$ to depend on the full $(s,a)$ and assumes no graph upfront, since we derive a graph from the support of $H^\pi$ after the fact. Older work on factored MDPs~\citep{kearns1999efficient, factor_guestrin2003efficient} and weakly coupled MDPs~\citep{weakly_mdp_meuleau1998solving} uses different forms of locality, typically asking for local rewards or independent transitions. The policy-dependent angle here parallels work that ties decay rates to policy regularization in single-agent control, but to our knowledge gives the first policy-dependent spectral certificate of value-locality for average-reward networked MARL. Function-approximation MARL methods~\citep{zhang2018fully,lowe2017multi} and independent learners~\citep{tan1993multi,matignon2012independent} attack scalability differently and do not certify locality. Appendix~\ref{app:rel} expands on this.

\section{Setup and preliminaries} \label{sec:setup}

Consider a system of $n$ agents. Each agent $i\in[n]$ has a finite state space $\Sset_i$ and a finite action space $\Aset_i$. The joint spaces are $\Sset=\prod_i \Sset_i$ and $\Aset=\prod_i \Aset_i$, with elements written $s=(s_1,\dots,s_n)$ and $a=(a_1,\dots,a_n)$, and $s_{-i}$, $a_{-i}$ for the profiles that exclude agent $i$.

At time $t$ the joint state is $s^t$, a joint action $a^t\sim\pi(\cdot\mid s^t)$ is drawn, and the next state $s^{t+1}\sim P(\cdot\mid s^t,a^t)$. We assume the kernel factors as $P(s'\mid s,a)=\prod_{i=1}^n P_i(s'_i\mid s,a)$, so that given $(s,a)$ the next-state coordinates are conditionally independent across agents. Each $P_i$ is still allowed to depend on the full $(s,a)$. This is more permissive than the typical networked-MARL assumption $P_i(s'_i\mid s_{N_i},a_i)$ used in \citep{qu_scalabale_marl_2020, lin_scalable_marl_stochastic_2021}, which fixes a graph $N_i$ ahead of time. We do not, since we want the locality structure to come out of the analysis rather than the modeling assumptions.

Policies are of product form $\pi(a\mid s)=\prod_{i=1}^n \pi_i(a_i\mid s_{O_i})$, where the observation scope $O_i\subseteq[n]$ is the set of state coordinates agent $i$'s policy actually depends on. Taking $O_i=[n]$ recovers globally conditioned product policies; smaller $O_i$ models agents that observe only locally. The kernel induced by $\pi$ is
\[
P^\pi(s'\mid s)=\sum_{a\in\Aset}\Big(\prod_{i=1}^n P_i(s'_i\mid s,a)\Big)\Big(\prod_{j=1}^n \pi_j(a_j\mid s_{O_j})\Big).
\]

\paragraph{Average reward and the Poisson equation.}
For the average-reward analysis we restrict to policies $\pi$ under which $P^\pi$ is irreducible on $\Sset$, so that it has a unique stationary distribution $d^\pi$. Given a per-step reward $r:\Sset\times\Aset\to\R$, write $r^\pi(s)=\sum_a r(s,a)\prod_k \pi_k(a_k\mid s_{O_k})$ for the expected one-step reward in state $s$ and $\bar r^\pi=\sum_s d^\pi(s)r^\pi(s)$ for the average reward under $\pi$. The object that plays the role of the value function in this setting is the \emph{bias function} $h^\pi$, which measures the long-run advantage of starting in $s$ relative to starting from stationarity and is defined up to a constant by the Poisson equation
\[
h^\pi - T^\pi h^\pi \ =\ r^\pi - \bar r^\pi, \qquad T^\pi f(s) := \E_{S'\sim P^\pi(\cdot\mid s)}[f(S')].
\]
This $h^\pi$ is used by policy-gradient and policy-improvement updates the same way the discounted value function is. When we say the system is value-local, we mean exactly that the coordinatewise oscillations of $h^\pi$ decay quickly in graph distance. Discounted-setting corollaries are in the appendix.

\paragraph{Coordinatewise oscillations.}
The natural way to quantify how much an agent's reward or value depends on every other agent is the coordinatewise oscillation. For a bounded $f:\Sset\to\R$, the $i$-oscillation
\[
\delta_i(f)\ =\ \sup_{x,y\in\Sset:\,x_{-i}=y_{-i}}|f(x)-f(y)|
\]
is the largest change $f$ can undergo when only coordinate $i$ changes; equivalently, it is the Lipschitz constant of $f$ under a single-coordinate Hamming change. Write $\delta(f)=(\delta_i(f))_{i=1}^n$ for the vector of oscillations and $\|\delta(f)\|_\infty=\max_i\delta_i(f)$ for its maximum. The seminorm $\|\delta(\cdot)\|_\infty$ is the standard object in Dobrushin-type arguments for interacting particle systems and Glauber dynamics~\citep{Dobruschin1968TheDO, Martinelli1999}; basic properties are in Appendix~\ref{app:oscillation_proof}. Spatial truncation works precisely when the entries of $\delta(h^\pi)$ are small at agents far from any source of reward variation.

\paragraph{Three sensitivity matrices.}
We measure one-step coupling with three $n\times n$ nonnegative matrices. The first two depend only on the environment; the third depends on the policy.
\begin{align*}
E^{\mathrm{s}}_{j\leftarrow i} &\ =\ \sup_{\substack{s,s'\in\Sset:\,s_{-i}=s'_{-i}\\a\in\Aset}}\TV\!\big(P_j(\cdot\mid s,a),\,P_j(\cdot\mid s',a)\big), &&\text{(state channel)}\\[2pt]
E^{\mathrm{a}}_{j\leftarrow i} &\ =\ \sup_{\substack{s\in\Sset\\a,a'\in\Aset:\,a_{-i}=a'_{-i}}}\TV\!\big(P_j(\cdot\mid s,a),\,P_j(\cdot\mid s,a')\big), &&\text{(action channel)}\\[2pt]
\Pi_{j\leftarrow i}(\pi) &\ =\ \sup_{s,s':\,s_{-i}=s'_{-i}}\TV\!\big(\pi_j(\cdot\mid s_{O_j}),\,\pi_j(\cdot\mid s'_{O_j})\big). &&\text{(policy reactivity)}
\end{align*}
In words, $E^{\mathrm s}_{j\leftarrow i}$ is the largest jump in $j$'s next-state law that a change of $s_i$ can cause with the action held fixed; $E^{\mathrm a}_{j\leftarrow i}$ is the corresponding quantity for an action change at $i$ with the state held fixed; and $\Pi_{j\leftarrow i}(\pi)$ is the largest jump in $j$'s action distribution that a change of $s_i$ can cause. The last is zero whenever $i\notin O_j$. The policy-induced closed-loop influence of $i$ on the next-state marginal of $j$ is
\[
C^\pi_{j\leftarrow i}=\sup_{s,s':\,s_{-i}=s'_{-i}}\TV\!\big(P^\pi_j(\cdot\mid s),P^\pi_j(\cdot\mid s')\big),
\qquad
P^\pi_j(\cdot\mid s)=\sum_{a}P_j(\cdot\mid s,a)\prod_k \pi_k(a_k\mid s_{O_k}),
\]
and $C^\pi$ is the matrix that prior work bounds by Dobrushin row sums. The next section decomposes it into the environment and policy pieces and replaces the row-sum bound with a spectral one.

\section{Policy-induced influence and locality} \label{sec:influence}

This section gives the decomposition of $C^\pi$ that powers everything else, the spectral condition that controls how the Poisson solution decays, and the resulting locality of the average-reward bias function. Proofs are in the appendix.

A perturbation of $s_i$ reaches the next-state marginal of $j$ along two routes in one step. The first is direct through the environment: even with the action fixed, a change in $s_i$ can move $P_j(\cdot\mid s,a)$. The second is indirect, through the policy and back into the environment, since changing $s_i$ shifts the action distribution of each agent $k$, and a change in $a_k$ can in turn move $j$'s next-state law. The direct route is bounded by $E^{\mathrm s}_{j\leftarrow i}$. The indirect route factors as $\sum_k E^{\mathrm a}_{j\leftarrow k}\Pi_{k\leftarrow i}(\pi)$, picking up the action-channel weight $E^{\mathrm a}_{j\leftarrow k}$ from the policy-action-state hop and the policy weight $\Pi_{k\leftarrow i}(\pi)$ from the state-policy hop.

\begin{proposition}[Decomposition of policy-induced influence]
\label{prop:decomposition}
For any product policy $\pi$ and any factorized synchronous dynamics on a finite state space, $C^\pi \preceq E^{\mathrm{s}} + E^{\mathrm{a}}\Pi(\pi)$ entrywise. Equivalently, $C^\pi_{j\leftarrow i}\le E^{\mathrm{s}}_{j\leftarrow i}+\sum_k E^{\mathrm{a}}_{j\leftarrow k}\Pi_{k\leftarrow i}(\pi)$ for every $i,j$.
\end{proposition}

The proof inserts an intermediate distribution that uses the new policy weights but the old kernel. It then bounds the two halves separately, once by a TV-convexity step (which gives the $E^{\mathrm s}$ term) and once by a maximal coupling on the actions (which gives the $E^{\mathrm a}\Pi$ term). Full details are in Appendix~\ref{app:proof-prop1}.

\paragraph{A two-agent example.}
Take $n=2$ binary agents, $\Sset_i=\Aset_i=\{0,1\}$. Let $P_1$ be constant (agent~$1$'s next state is independent of everything), and let agent~$2$'s next state copy agent~$1$'s action: $P_2(s'_2=1\mid s,a)=\mathbf 1\{a_1=1\}$. Then $E^{\mathrm s}=0$ and $E^{\mathrm a}_{2\leftarrow 1}=1$ with all other entries of $E^{\mathrm a}$ equal to zero. An action-supremum bound stops here and reports the system as globally coupled, since $E^{\mathrm a}_{2\leftarrow 1}=1$ is as large as it can be. Now suppose agent~$1$'s policy changes by at most $\alpha$ in TV when $s_1$ flips, and that $\pi_1$ does not observe $s_2$ (so $\Pi_{1\leftarrow 2}=0$). Then $\Pi_{1\leftarrow 1}(\pi)=\alpha$, and Proposition~\ref{prop:decomposition} gives
\[
C^\pi_{2\leftarrow 1}\ \le\ E^{\mathrm s}_{2\leftarrow 1}+E^{\mathrm a}_{2\leftarrow 1}\Pi_{1\leftarrow 1}+E^{\mathrm a}_{2\leftarrow 2}\Pi_{2\leftarrow 1}\ =\ \alpha.
\]
The closed-loop coupling is as small as the policy. The same bound handles the failure mode: if $\pi_1$ is sharp in $s_1$ ($\alpha$ near $1$), it returns $\alpha$ and recovers the action-supremum answer. Smoothing the policy only helps when there is policy reactivity to smooth out. Appendix~\ref{app:policy-ex} works through a complementary instance.

We now convert the entrywise decomposition into a one-step contraction.

\begin{lemma}[Oscillation bound via $H^\pi$]\label{lem:one-step}
Let $T^\pi f(s)=\E[f(S_{t+1})\mid S_t=s]$ under the synchronous update with product policy and factorized kernel. Then $\delta_i(T^\pi f)\le\sum_j H^\pi_{j\leftarrow i}\delta_j(f)$ for every $i$ and every bounded $f$, where $H^\pi\coloneqq E^{\mathrm s}+E^{\mathrm a}\Pi(\pi)$. In vector form, $\delta(T^\pi f)\le (H^\pi)^\top\delta(f)$.
\end{lemma}

Iterating Lemma~\ref{lem:one-step} turns a spectral bound on $H^\pi$ into global decay of $T^\pi$, which in turn bounds the Poisson solution.

\begin{theorem}[Policy-uniform contraction and Poisson decay]
\label{thm:poisson}
Let $\Pi_{\mathrm{pol}}$ be a compact class of product policies, set $H^\pi\coloneqq E^{\mathrm s}+E^{\mathrm a}\Pi(\pi)$ and $\lambda_\star\coloneqq\sup_{\pi\in\Pi_{\mathrm{pol}}}\rho(H^\pi)$, and assume $\lambda_\star<1$. For every $\pi\in\Pi_{\mathrm{pol}}$, every bounded $f$, and every $t\ge 0$,
\[
\delta\big((T^\pi)^t f\big)\le \big((H^\pi)^\top\big)^t\delta(f),
\]
and for every $\bar\lambda\in(\lambda_\star,1)$ there exists $C_{\bar\lambda}\ge 1$ independent of $t,\pi$ such that $\|\delta((T^\pi)^t f)\|_\infty\le C_{\bar\lambda}\bar\lambda^t\|\delta(f)\|_\infty$. If $P^\pi$ is irreducible for every $\pi\in\Pi_{\mathrm{pol}}$, with stationary distribution $d^\pi$, then the Poisson equation $h^\pi-T^\pi h^\pi=r^\pi-\bar r^\pi$ has a solution unique up to an additive constant, and every solution satisfies
\[
\delta(h^\pi)\le \sum_{t=0}^\infty \big((H^\pi)^\top\big)^t\delta(r^\pi) \le \big(I-(H^\pi)^\top\big)^{-1}\delta(r^\pi).
\]
\end{theorem}

Theorem~\ref{thm:poisson} does not assume any agent graph upfront. A graph does emerge, though, through the directed support graph $G_H^\pi$ of $H^\pi$, in which $i\to j$ is an edge whenever $H^\pi_{j\leftarrow i}>0$ (agent $i$'s state directly affects agent $j$'s next state in one step under $\pi$). The matrix power $((H^\pi)^\top)^t$ then has a clean path interpretation: $((H^\pi)^\top)^t_{ij}$ is a path-weighted sum, where each contributing path runs from $i$ to $j$ in exactly $t$ edges of $G_H^\pi$ and has weight equal to the product of $H^\pi$ entries along it. Three consequences follow.
\begin{itemize}[leftmargin=*,topsep=2pt,itemsep=1pt]
\item \emph{Path accumulation.} $\delta_i(h^\pi)$ is bounded by a sum over all directed paths in $G_H^\pi$ starting at $i$ and ending at any agent $j$ where the reward varies, weighted by $H^\pi$ along the path.
\item \emph{Exponential attenuation.} The contribution of paths of length $t$ is $O(\bar\lambda^t)$, so agents far downstream of $i$ in $G_H^\pi$ matter exponentially less.
\item \emph{Truncation error.} Ignoring agents more than $\kappa$ hops downstream of $i$ drops the tail of the Neumann series, which is bounded by $\tfrac{C}{1-\lambda}\bar\lambda^{\kappa+1}\|\delta(r^\pi)\|_\infty$.
\end{itemize}
If $E^{\mathrm s},E^{\mathrm a},\Pi(\pi)$ are each local with respect to some underlying graph $G$, then $G_H^\pi$ is a finite-radius closure of $G$. It need not equal $G$, because $E^{\mathrm a}\Pi(\pi)$ can create new closed-loop edges. An action of agent $k$ that affects $j$, combined with a policy at $k$ that reacts to $s_i$, produces a closed-loop edge $i\to j$ even when $i$ and $j$ are not graph-neighbors in $G$.

Theorem~\ref{thm:poisson} recovers and strengthens the locality results of \citep{qu_scalabale_marl_2020, lin_scalable_marl_stochastic_2021}: the spectral-radius condition is strictly weaker than the row-sum one on any single comparison matrix, and the policy-dependent matrix $H^\pi$ can be much smaller than the policy-independent action-supremum bounds those works use, so $\rho(H^\pi)<1$ can hold in regimes where their row-sum conditions fail. The two-agent example above is one such case; Appendix~\ref{app:sparse} works out the formal embedding. The reward enters Theorem~\ref{thm:poisson} only through $\delta(r^\pi)$, so a reward that decomposes locally produces a sparse $\delta(r^\pi)$ and tightens the bound automatically; no separate local-reward assumption is needed.

\paragraph{A small numerical instance of the gap.}
The gap between $\rho(H^\pi)$ and $\|H^\pi\|_\infty$ can be large even on small matrices, and it is easy to build closed-loop matrices that our condition certifies but the Dobrushin row-sum bound does not. For $n=2$ agents, take
\[
H^\pi\ =\ \begin{pmatrix} 0.6 & 0 \\ 0.6 & 0.6\end{pmatrix}.
\]
With our convention $H^\pi_{j\leftarrow i}=H^\pi[j,i]$, this models a system in which agent~$1$ influences agent~$2$'s next state ($H^\pi_{2\leftarrow 1}=0.6$) but agent~$2$ does not influence agent~$1$ ($H^\pi_{1\leftarrow 2}=0$). The row sums are $0.6$ and $1.2$, so $\|H^\pi\|_\infty=1.2$ and the row-sum condition is violated. The eigenvalues are both $0.6$, so $\rho(H^\pi)=0.6$ and Theorem~\ref{thm:poisson} certifies locality with decay rate arbitrarily close to $0.6$. A $\kappa$-hop truncation costs error $O(0.6^{\kappa+1})$, so $\kappa=10$ already drives the truncation error below $10^{-2}$ while the row-sum diagnostic offers no guarantee at any $\kappa$. The asymmetry is what creates the gap: closed-loop influence flows one way ($1\to 2$) and decays in one hop in the reverse direction, even though one row sum is above one.

\begin{figure}[h]
\centering
\begin{tikzpicture}[>={Stealth[length=2.5mm]},every node/.style={font=\small}]
\node[circle,draw,fill=blue!10,minimum size=9mm,inner sep=0] (si) at (0,1.8) {$s_i$};
\node[circle,draw,fill=blue!10,minimum size=9mm,inner sep=0] (sjp) at (7,1.8) {$s'_j$};
\draw[->, line width=0.7pt] (si) -- node[above,font=\small]{direct route: $E^{\mathrm s}_{j\leftarrow i}$} (sjp);
\node[circle,draw,fill=orange!20,minimum size=9mm,inner sep=0] (ak) at (3.5,0) {$a_k$};
\draw[->, line width=0.7pt] (si) to[bend right=20] node[below left=-2pt,font=\small]{$\Pi_{k\leftarrow i}(\pi)$} (ak);
\draw[->, line width=0.7pt] (ak) to[bend right=20] node[below right=-2pt,font=\small]{$E^{\mathrm a}_{j\leftarrow k}$} (sjp);
\node[font=\small] at (3.5,-0.8) {indirect route, summed over $k$: $\sum_k E^{\mathrm a}_{j\leftarrow k}\,\Pi_{k\leftarrow i}(\pi)$};
\end{tikzpicture}
\caption{The decomposition $C^\pi_{j\leftarrow i}\le E^{\mathrm s}_{j\leftarrow i}+\sum_k E^{\mathrm a}_{j\leftarrow k}\Pi_{k\leftarrow i}(\pi)$, written as two routes from a perturbation of $s_i$ to the next-state marginal of $j$. The direct route through the environment has weight $E^{\mathrm s}_{j\leftarrow i}$. The indirect route passes through each intermediate action $a_k$, picking up a factor $\Pi_{k\leftarrow i}(\pi)$ for the state-to-action hop and $E^{\mathrm a}_{j\leftarrow k}$ for the action-to-next-state hop, then sums over $k$.}
\label{fig:support}
\end{figure}

\paragraph{Discounted setting.}
For the standard $\gamma$-discounted Bellman operator $T^\pi_\gamma$, the one-step bound becomes $\delta(T^\pi_\gamma f)\le \gamma(H^\pi)^\top\delta(f)$, so locality is certified whenever $\gamma\rho(H^\pi)<1$. The decay rate is set jointly by the temporal factor $\gamma$ and the structural factor $\rho(H^\pi)$. With $\gamma=0.99$ and $\rho(H^\pi)=0.5$, for instance, taking $\tilde\lambda=0.5$ as a certificate gives a neighborhood of order $7$ to reach error $10^{-2}$, against roughly $458$ from a $\gamma$-only bound. The same analysis carries over to the asynchronous Glauber-style update; see Appendix~\ref{app:async}.

\section{Controlling policy sensitivity via softmax temperature}
\label{sec:softmax}

The decay rate in Theorem~\ref{thm:poisson} is $\rho(H^\pi)=\rho(E^{\mathrm s}+E^{\mathrm a}\Pi(\pi))$. The environment terms $E^{\mathrm s},E^{\mathrm a}$ are fixed once the MDP is given. The policy term $\Pi(\pi)$ is set by the policy class, so it is what a learning algorithm can change. A natural way to keep $\Pi(\pi)$ small is to restrict the policy class to state-to-action maps that are smooth. The temperature-$\tau$ softmax class~\citep{Geist2019ATO, HaarnojaZAL18} is the standard such class, and it is the policy class used in entropy-regularized control, soft actor-critic, and KL-proximal updates, where $\tau$ is a hyperparameter of the algorithm. We show in this section that the same $\tau$ that controls the reward-versus-entropy tradeoff in those methods also bounds $\Pi(\pi)$ entrywise, and so bounds the locality certificate of Theorem~\ref{thm:poisson}. All proofs are in Appendix~\ref{app:softmax-Pi}.

\paragraph{Setup.}
The entropy-regularized average-reward objective is
\[
J_\tau(\pi)\ =\ \E_{s\sim d^\pi,\ a\sim \pi(\cdot\mid s)}\!\Big[r(s,a)-\tau\sum_{k=1}^n \log\pi_k(a_k\mid s_{O_k})\Big],
\]
the average-reward analogue of the discounted soft-Bellman objective. We study the temperature-$\tau$ softmax policy class commonly used in entropy-regularized and KL-proximal methods.

\begin{definition}[Softmax policy]
A policy $\pi$ is a temperature-$\tau$ softmax policy if for each agent $k$ there is a logit function $g_k:\Sset_{O_k}\times\Aset_k\to\R$ such that $\pi_k(a_k\mid s_{O_k})\propto \exp(g_k(s_{O_k},a_k)/\tau)$.
\end{definition}

How reactive a softmax policy is to its state input is controlled by $\tau$ together with how reactive the logit itself is.

\begin{definition}[Logit Lipschitz constant]
For a logit function $g_k$, the one-coordinate logit constant with respect to state $i$ is $L_{k\leftarrow i}\ :=\ \sup_{s_{-i}=s'_{-i}}\|g_k(s_{O_k},\cdot)-g_k(s'_{O_k},\cdot)\|_\infty$.
\end{definition}

$L_{k\leftarrow i}$ is a property of the logit parameterization alone and does not involve $\tau$. For a linear logit $g_k(s,a)=\langle w_{k,a},\phi(s)\rangle$ with features $\phi$ of single-coordinate oscillation $\delta_i(\phi)$, one can take $L_{k\leftarrow i}\le \sup_a 2\|w_{k,a}\|\delta_i(\phi)$, so $L$ is set entirely by the feature design and the weight magnitudes and can be precomputed.

\begin{lemma}[Softmax temperature controls $\Pi(\pi)$]\label{lem:Pi-softmax}
For a temperature-$\tau$ softmax policy, $\Pi_{k\leftarrow i}(\pi)\ \le\ \min\{1,\,L_{k\leftarrow i}/(2\tau)\}$ for all $k,i\in[n]$.
\end{lemma}

The proof bounds the TV distance between two softmax distributions by a sigmoidal function of the difference of their logits, then linearizes at zero with Lipschitz constant $1/(2\tau)$. The cap at $1$ is the trivial TV bound. Writing $L=[L_{k\leftarrow i}]$ for the matrix of logit constants and applying the lemma entrywise gives $\Pi(\pi)\preceq \min\{\mathbf 1,L/(2\tau)\}$. The spectral radius $\rho(M)$ of a nonnegative matrix $M$ is entrywise monotone (Perron--Frobenius), so an entrywise upper bound on $H^\pi$ gives an upper bound on $\rho(H^\pi)$:
\[
H^\pi \ \preceq\ E^{\mathrm s}+E^{\mathrm a}\,\frac{L}{2\tau}
\quad\Longrightarrow\quad
\rho(H^\pi)\ \le\ \rho\!\left(E^{\mathrm s}+E^{\mathrm a}\,\frac{L}{2\tau}\right).
\]
The right-hand side is the quantity a practitioner actually computes. It is built from the (known) environment matrices $E^{\mathrm s},E^{\mathrm a}$ and the (designed) logit constants $L$, and its spectral radius is monotone nonincreasing in $\tau$.

\paragraph{Locality--optimality tradeoff.}
The behavior at the two ends of the temperature range is easy to describe. As $\tau\to\infty$, the softmax policy approaches uniform on every state, $\Pi(\pi)\to 0$, and $H^\pi\to E^{\mathrm s}$. Whenever $\rho(E^{\mathrm s})<1$, a high enough $\tau$ certifies locality regardless of how strong the action channel is. The cost is that the policy is far from greedy and the objective $J_\tau$ deviates from the unregularized average reward by $O(\tau\log|\Aset|)$. As $\tau\to 0$ the policy becomes deterministic, $\Pi(\pi)$ saturates near $1$ wherever the logit ordering is sensitive to the state, and $\rho(H^\pi)$ can cross one and break the certificate.

In an algorithmic loop this gives a clean recipe. Pick a minimum temperature $\tau_{\min}$ at which $\rho(E^{\mathrm s}+E^{\mathrm a}L/(2\tau_{\min}))<1$ holds with margin, run the entropy-regularized improvement at $\tau=\tau_{\min}$ with a truncation radius $\kappa$ chosen to push the bias below the desired accuracy, and either stop there (accepting the regularization bias) or anneal $\tau$ downward in stages, increasing $\kappa$ at each stage to compensate. Appendix~\ref{app:nonlocal-ex} works through a quantitative instance.

\section{An oracle framework for localized evaluation and block-coordinate improvement}
\label{sec:algo}

We now use the decay result of Section~\ref{sec:influence} to give a deterministic oracle guarantee for a block-coordinate KL-proximal improvement template. The section deliberately separates the structural effect of locality from the orthogonal issues of statistical estimation and function approximation, since the latter can be addressed by standard tools once the structural picture is clear. The only object in our framework that is guaranteed to be locally computable is the \emph{certificate} below; the truncated Poisson surrogate is an oracle quantity used to bound bias.

Recall the support graph $G_H^\pi$ of $H^\pi=E^{\mathrm s}+E^{\mathrm a}\Pi(\pi)$, with edge $i\to j$ whenever $H^\pi_{j\leftarrow i}>0$. When $H^\pi$ is sparse, the $\kappa$-hop neighborhoods in $G_H^\pi$ are small.

\paragraph{Phase 1: locality certificate.}
Set $b^\pi=\delta(r^\pi)$. Theorem~\ref{thm:poisson} gives $\delta(h^\pi)\le \sum_{t\ge 0}((H^\pi)^\top)^t b^\pi$. Truncating this Neumann series at depth $\kappa$ yields the computable certificate
\[
\widehat\delta^{(\kappa)} \coloneqq \sum_{t=0}^{\kappa}\big((H^\pi)^\top\big)^t b^\pi.
\]
The $i$-th component $\widehat\delta^{(\kappa)}_i$ is a sum over directed paths in $G_H^\pi$ of length at most $\kappa$ starting at $i$ (weighted by $H^\pi$ entries along the path) and accumulating $b^\pi_j$ at each endpoint $j$. It depends only on the $\kappa$-hop \emph{out-ball} of $i$ in $G_H^\pi$ (the agents reachable from $i$ in at most $\kappa$ directed edges) and is computable in $\kappa$ rounds of message passing along the reverse edges of that ball. For the bias analysis we also use the truncated Poisson surrogate $\widehat h_\kappa^\pi \coloneqq \sum_{t=0}^{\kappa}(T^\pi)^t(r^\pi-\bar r^\pi)$.

\begin{theorem}[Localized certificate and truncation bias]
\label{thm:local-eval}
Let $\pi$ be a product policy with $P^\pi$ irreducible, and suppose $\|((H^\pi)^\top)^t\|_{\infty\to\infty}\le C\lambda^t$ for some $C\ge 1$, $\lambda\in(0,1)$. Let $R^{(\kappa)}=\sum_{t>\kappa}((H^\pi)^\top)^t b^\pi$. Every solution $h^\pi$ of the Poisson equation satisfies $\delta(h^\pi)\le \widehat\delta^{(\kappa)}+R^{(\kappa)}$ with $\|R^{(\kappa)}\|_\infty\le \tfrac{C}{1-\lambda}\lambda^{\kappa+1}\|b^\pi\|_\infty$, and the oracle bias of $\widehat h_\kappa^\pi$ satisfies
\[
B_\kappa^\pi
\coloneqq
\inf_{c\in\R}\|\widehat h_\kappa^\pi-h^\pi-c\mathbf 1\|_\infty
\ \le\
\tfrac12\,\|R^{(\kappa)}\|_1
\ \le\
\tfrac{nC}{2(1-\lambda)}\lambda^{\kappa+1}\|b^\pi\|_\infty.
\]
\end{theorem}

The two bounds play different roles. $\|R^{(\kappa)}\|_\infty$ is the residual error of the certificate the agent actually computes. $B_\kappa^\pi$ is the oracle-side bias of the truncated Poisson surrogate, which appears only in the proof of the improvement step below. Both decay at the same rate $\lambda^{\kappa+1}$.

\paragraph{Phase 2: block KL-prox improvement.}
We use the surrogate $\widehat h_\kappa^\pi$ to drive a one-agent policy update. The KL-proximal (or mirror-descent) update is a standard step in entropy-regularized policy gradient, NPG, and TRPO-style methods: it picks a new policy that maximizes the predicted improvement (a linearized advantage) minus $\eta$ times the KL divergence from the old policy, where the temperature $\eta>0$ controls how aggressive the step is. We apply this update one agent at a time. This per-step $\eta$ is separate from the softmax temperature $\tau$ of Section~\ref{sec:softmax}, which is a property of the policy class. Fix a baseline product policy $\pi$ and an agent $k$. The exact advantage and its truncated counterpart are
\[
A^\pi(s,a)=r(s,a)-\bar r^\pi+\sum_y P(y\mid s,a)h^\pi(y)-h^\pi(s),
\qquad
\widehat A_\kappa^\pi=A^\pi\big|_{h^\pi\to \widehat h_\kappa^\pi}.
\]
The exact and approximate block logits for agent $k$ are
\[
g_{k,\star}^\pi(s,a_k)=\E_{a_{-k}\sim \pi_{-k}(\cdot\mid s)}\!\big[A^\pi(s,(a_k,a_{-k}))\big],
\qquad
\widehat g_{k,\kappa}^\pi(s,a_k)=\E_{a_{-k}\sim \pi_{-k}(\cdot\mid s)}\!\big[\widehat A_\kappa^\pi(s,(a_k,a_{-k}))\big],
\]
and the KL-prox update is $\mu_k(\cdot\mid s)\propto \pi_k(\cdot\mid s)\exp(\widehat g_{k,\kappa}^\pi(s,\cdot)/\eta)$ with $\mu_{-k}=\pi_{-k}$.

\begin{theorem}[One-block oracle improvement with logit error]
\label{thm:block-improve}
Under the assumptions of Theorem~\ref{thm:local-eval}, and assuming additionally that $\pi$ has full action support at every state so that $P^\mu$ inherits irreducibility from $P^\pi$ (the KL-prox update preserves the action support of $\pi$), the update $\mu$ above satisfies
\[
\bar r(\mu)-\bar r(\pi)
\ \ge\
\eta\,\E_{S\sim d^\mu}\!\Big[\KL\big(\mu_k(\cdot\mid S)\,\big\|\,\pi_k(\cdot\mid S)\big)\Big]
\ -\
\frac{2nC}{1-\lambda}\lambda^{\kappa+1}\|b^\pi\|_\infty.
\]
\end{theorem}

The same Neumann tail $\lambda^{\kappa+1}$ appears in the certificate of Theorem~\ref{thm:local-eval} and in the improvement bound, so a single truncation radius $\kappa$ controls both. The bound is additive in the KL-prox step (first term) and the truncation bias (second term), so the proximal temperature $\eta$ and the radius $\kappa$ can be tuned independently. Algorithm~\ref{alg:lpi} interleaves the two phases over outer iterations and inner agents, recomputing $H$ and $b$ at every step. Using stale $(H,b)$ over a full pass is also valid at the cost of a staleness term we omit.

\paragraph{What is local, what is oracle.}
The framework separates two layers of approximation that scalable MARL methods often conflate. The certificate $\widehat\delta^{(\kappa)}$ is genuinely local: each $\widehat\delta^{(\kappa)}_i$ depends only on entries of $H^\pi$ and of $b^\pi$ supported on the $\kappa$-hop out-ball of agent $i$ in $G_H^\pi$. The surrogate $\widehat h_\kappa^\pi$ and the truncated logit $\widehat g_{k,\kappa}^\pi$, by contrast, are oracle objects, since in the general model they depend on the full reward vector and on global action sums. The improvement bound treats them as accessible because the analysis only compares the truncated logit to the exact one. Making the surrogate locally computable needs either (i) restricted observation scopes $O_i$ that turn $r^\pi$ and the action expectations into $\kappa$-local sums, or (ii) function approximation of $\widehat h_\kappa^\pi$ on a parametric family with locality structure, in the spirit of the localized critic in~\citep{qu_scalabale_marl_2020, lin_scalable_marl_stochastic_2021}. The analysis here gives the structural part of the bound; an additional approximation-error term then enters additively.

\begin{algorithm}[H]
\caption{Oracle block-coordinate improvement with locality certificates}
\makeatletter
\edef\@currentlabel{\thealgorithm}%
\protected@write\@auxout{}{\string\newlabel{alg:lpi}{{\@currentlabel}{\thepage}{}{algorithm.\@currentlabel}{}}}%
\makeatother
\begin{algorithmic}[1]
\STATE \textbf{Input:} baseline $\pi^{(0)}$, radius $\kappa$, prox temperature $\eta$
\FOR{$\ell=0,1,\dots$}
   \STATE $\pi^{(\ell,0)}\leftarrow \pi^{(\ell)}$
   \FOR{$k=1,\dots,n$}
      \STATE compute $b^{(\ell,k-1)}=\delta(r^{\pi^{(\ell,k-1)}})$ and $H^{(\ell,k-1)}=E^{\mathrm s}+E^{\mathrm a}\Pi(\pi^{(\ell,k-1)})$
      \STATE certificate $\widehat\delta^{(\kappa,\ell,k-1)}=\sum_{t=0}^{\kappa}\big((H^{(\ell,k-1)})^\top\big)^t b^{(\ell,k-1)}$\hfill ($\kappa$ rounds message-pass)
      \STATE oracle surrogate $\widehat h_\kappa^{(\ell,k-1)}=\sum_{t=0}^{\kappa}(T^{\pi^{(\ell,k-1)}})^t(r^{\pi^{(\ell,k-1)}}-\bar r^{\pi^{(\ell,k-1)}})$
      \STATE block logit $\widehat g_{k,\kappa}^{(\ell,k-1)}$ via Phase~2 with $h^\pi\to \widehat h_\kappa^{(\ell,k-1)}$
      \STATE update $\pi_k^{(\ell,k)}(\cdot\mid s)\propto \pi_k^{(\ell,k-1)}(\cdot\mid s)\exp(\widehat g_{k,\kappa}^{(\ell,k-1)}(s,\cdot)/\eta)$;\ $\pi_j^{(\ell,k)}=\pi_j^{(\ell,k-1)}$ for $j\ne k$
   \ENDFOR
   \STATE $\pi^{(\ell+1)}\leftarrow \pi^{(\ell,n)}$
\ENDFOR
\end{algorithmic}
\end{algorithm}

\section{Conclusion}

We gave a policy-dependent spectral certificate of value-locality, $\rho(H^\pi)<1$ for $H^\pi=E^{\mathrm s}+E^{\mathrm a}\Pi(\pi)$, that is strictly weaker than the row-sum condition on the same matrix and that applies in regimes where the policy-independent action-supremum bounds used in prior Dobrushin-style work cannot. The decomposition splits the closed-loop influence into an environment piece ($E^{\mathrm s},E^{\mathrm a}$) and a policy piece ($\Pi(\pi)$). The split makes precise the intuition that a smooth policy can neutralize a strong action channel, since $E^{\mathrm a}\Pi(\pi)$ shrinks linearly in policy smoothness even when $E^{\mathrm a}$ is large. For temperature-$\tau$ softmax policies the bound $\Pi(\pi)\le L/(2\tau)$ makes $\tau$ a direct way to tighten the certificate, and the same Neumann tail $\lambda^{\kappa+1}$ that governs the decay of the bias function also governs the truncation bias of an oracle block-coordinate KL-prox improvement template, whose decentralized realization is the natural next step.

\bibliography{l4dc2026-sample}

\clearpage
\appendix

\section{Related work}\label{app:rel}
Our work addresses the curse of dimensionality in multi-agent reinforcement learning (MARL). The challenge is that the global state and action spaces ($\Sset = \prod_i \Sset_i$, $\Aset = \prod_i \Aset_i$) grow exponentially with the number of agents, $n$. This problem falls into the category of ``succinctly described" MDPs \citep{complexity_blondel2000survey}, which are known to be computationally intractable in the general case, even with structural assumptions \citep{complexity_papadimitriou1999complexity}.

To achieve scalability, researchers have explored various structural assumptions.
One common approach in general MARL is to use function approximation \citep{zhang2018fully, lowe2017multi} or assume independent learners \citep{tan1993multi}, though the latter can suffer from non-stationarity \citep{matignon2012independent}. Other related areas include \textit{Factored MDPs} \citep{kearns1999efficient, factor_guestrin2003efficient}, which assume local states but typically a global action, and \textit{Weakly Coupled MDPs} \citep{weakly_mdp_meuleau1998solving}, which assume agents' transitions are independent and coupling only occurs through the reward. Our work differs from these by focusing on systems with local, coupled transitions, which is common in networked systems.

This paper is most related to the line of work on scalable Networked MARL \citep{qu_scalabale_marl_2020, qu2019scalable, lin_scalable_marl_stochastic_2021}. These foundational papers were the first to show that if the system's local transitions and rewards depend only on a local graph neighborhood, the system exhibits an Exponential Decay Property (EDP). This property, where influence decays exponentially with graph distance, justifies $\kappa$-hop truncation and enables scalable algorithms.

In the average-reward setting these works establish the EDP via a Dobrushin coupling condition~\citep{qu_scalabale_marl_2020}, which bounds environment influence by a supremum over joint actions and asks $\|C\|_\infty<1$ for the resulting policy-independent matrix $C$. This bound is tight only when the policy realizes the worst-case action; it fails to certify locality whenever the action channel is strong but the policy in use does not excite it. The present paper replaces this row-sum condition with the spectral condition $\rho(H^\pi)<1$ for $H^\pi=E^{\mathrm s}+E^{\mathrm a}\Pi(\pi)$, which is policy-dependent, recovers the prior condition as a special case, and certifies locality in regimes where the prior one cannot.

\section{Proofs from Section~\protect\ref{sec:influence}}
This appendix contains complete proofs of the results in the main text and an extension to asynchronous updates. All total variation distances are normalized: for probability measures $\mu$ and $\nu$ on a finite set, $\TV(\mu,\nu)=\tfrac12\sum_x|\mu(x)-\nu(x)|$. For a bounded function $f$ on a product space, the coordinatewise oscillations are $\delta_i(f)=\sup_{x,y\in\Sset,\ x_{-i}=y_{-i}} |f(x)-f(y)|$ and $\delta(f)$ is the vector with entries $\delta_i(f)$.

\subsection{Proof of Proposition~\protect\ref{prop:decomposition}} \label{app:proof-prop1}

Given $i,~j\in [n]$ and two states $s,s'$ that agree off coordinate $i$, define
\[
\mu_j=\sum_{a\in\Aset}\Big(\prod_{k} \pi_k(a_k\mid s_{O_k})\Big) P_j(\cdot\mid s,a),
\qquad
\nu_j=\sum_{a\in\Aset}\Big(\prod_{k} \pi_k(a_k\mid s'_{O_k})\Big) P_j(\cdot\mid s',a).
\]
We claim
\[
\TV(\mu_j,\nu_j)\ \le\ E^{\mathrm{s}}_{j\leftarrow i}\ +\ \sum_{k} E^{\mathrm{a}}_{j\leftarrow k}\,\Pi_{k\leftarrow i}(\pi).
\]
The decomposition uses the triangle inequality and a coupling argument.

Define an intermediate distribution
\[
\tilde\mu_j=\sum_{a\in\Aset}\Big(\prod_{k} \pi_k(a_k\mid s'_{O_k})\Big) P_j(\cdot\mid s,a).
\]
By the triangle inequality,
\[
\TV(\mu_j,\nu_j)\ \le\ \TV(\mu_j,\tilde\mu_j)\ +\ \TV(\tilde\mu_j,\nu_j).
\]

We bound $\TV(\tilde\mu_j,\nu_j)$ first. The measures $\tilde\mu_j$ and $\nu_j$ are mixtures with the same mixing weights $w(a)=\prod_k \pi_k(a_k\mid s'_{O_k})$. By convexity of total variation,
\begin{align*}
\TV(\tilde\mu_j,\nu_j)
&\le
\sum_{a\in\Aset} w(a)\,\TV\big(P_j(\cdot\mid s,a),P_j(\cdot\mid s',a)\big) \\
&\le
\sup_{a\in\Aset} \TV\big(P_j(\cdot\mid s,a),P_j(\cdot\mid s',a)\big)
\ \le\ E^{\mathrm{s}}_{j\leftarrow i}.
\end{align*}

We now bound $\TV(\mu_j,\tilde\mu_j)$. This is the difference between two mixtures with the same components and different mixing weights $p(a\mid s)=\prod_k \pi_k(a_k\mid s_{O_k})$ and $p(a\mid s')=\prod_k \pi_k(a_k\mid s'_{O_k})$, but the components $P_j(\cdot\mid s,a)$ can be far apart for different $a$. We upper bound this difference by a coupling that changes actions along one coordinate at a time and uses the action sensitivity of the kernel.

Let $(A,A')$ be a coupling of $p(\cdot\mid s)$ and $p(\cdot\mid s')$ constructed as follows. For each agent $k$, couple the marginals $\pi_k(\cdot\mid s_{O_k})$ and $\pi_k(\cdot\mid s'_{O_k})$ by a maximal coupling so that
\[
\mathbb{P}[A_k\neq A'_k]
\ =\
\TV\big(\pi_k(\cdot\mid s_{O_k}),\pi_k(\cdot\mid s'_{O_k})\big)
\ \le\
\Pi_{k\leftarrow i}(\pi).
\]
Take these couplings independent across $k$, which is possible because the policy is of product form and the coordinates are independent under each product marginal.

Conditional on $(A,A')$, couple next–state coordinates for agent $j$ by a maximal coupling of $P_j(\cdot\mid s,A)$ and $P_j(\cdot\mid s,A')$. Then
\[
\mathbb{P}[S'^{(j)}\neq \tilde S'^{(j)} \mid A,A']\ =\ \TV\big(P_j(\cdot\mid s,A),P_j(\cdot\mid s,A')\big).
\]
Taking expectation,
\[
\TV(\mu_j,\tilde\mu_j)
\ \le\
\mathbb{E}\,\TV\big(P_j(\cdot\mid s,A),P_j(\cdot\mid s,A')\big).
\]
For any $a,a'$ differing on a set $D\subseteq[n]$ of action coordinates, the triangle inequality and the definition of $E^{\mathrm{a}}$ imply
\[
\TV\big(P_j(\cdot\mid s,a),P_j(\cdot\mid s,a')\big)
\ \le\ \sum_{k\in D} E^{\mathrm{a}}_{j\leftarrow k}.
\]
Applying this bound with the random pair $(A,A')$ and taking expectations yields
\[
\TV(\mu_j,\tilde\mu_j)
\ \le\ \sum_{k=1}^n E^{\mathrm{a}}_{j\leftarrow k}\,\mathbb{P}[A_k\neq A'_k]
\ \le\ \sum_{k=1}^n E^{\mathrm{a}}_{j\leftarrow k}\,\Pi_{k\leftarrow i}(\pi),
\]
where the last inequality uses $\mathbb P[A_k\neq A'_k]\le \Pi_{k\leftarrow i}(\pi)$ from the maximal coupling above.
Combining the two pieces finishes the proof:
\[
\TV(\mu_j,\nu_j)\ \le\ E^{\mathrm{s}}_{j\leftarrow i}\ +\ \sum_{k=1}^n E^{\mathrm{a}}_{j\leftarrow k}\,\Pi_{k\leftarrow i}(\pi).
\]
By the definition of $C^\pi_{j\leftarrow i}$ as a supremum over $s,s'$ with $s_{-i}=s'_{-i}$, the same bound holds for $C^\pi_{j\leftarrow i}$.

\subsection{Proof of Lemma~\protect\ref{lem:one-step}}

Fix $i$ and two states $s,s'$ with $s_{-i}=s'_{-i}$. We will couple one synchronous step from $s$ and from $s'$.

First couple the actions. For each $k\in[n]$, let $(A_k,A'_k)$ be a maximal coupling of
$\pi_k(\cdot\mid s_{O_k})$ and $\pi_k(\cdot\mid s'_{O_k})$, chosen independently across $k$.
Then $(A,A')=( (A_k)_k,(A'_k)_k )$ is a coupling of the product measures
$\pi(\cdot\mid s)$ and $\pi(\cdot\mid s')$, and
\[
\mathbb P[A_k\neq A'_k]\ =\ \mathrm{TV}\!\big(\pi_k(\cdot\mid s_{O_k}),\pi_k(\cdot\mid s'_{O_k})\big)
\ \le\ \Pi(\pi)_{k\leftarrow i}.
\]

Next, conditional on $(A,A')=(a,a')$, draw the next state coordinates independently across $j$:
let $X_j\sim P_j(\cdot\mid s,a)$ and $Y_j\sim P_j(\cdot\mid s',a')$,
and couple $(X_j,Y_j)$ by a maximal coupling of these two marginals.
By construction,
\[
\mathbb P[X_j\neq Y_j \mid A=a, A'=a']\ =\ \mathrm{TV}\!\big(P_j(\cdot\mid s,a),P_j(\cdot\mid s',a')\big).
\]
Taking expectation over $(A,A')$ gives
\[
\mathbb P[X_j\neq Y_j]\ =\ \mathbb E_{A,A'}\Big[\mathrm{TV}\!\big(P_j(\cdot\mid s,A),P_j(\cdot\mid s',A')\big)\Big].
\]

For any $a,a'$, apply the triangle inequality:
\[
\mathrm{TV}\!\big(P_j(\cdot\mid s,a),P_j(\cdot\mid s',a')\big)
\ \le\
\underbrace{\mathrm{TV}\!\big(P_j(\cdot\mid s,a),P_j(\cdot\mid s,a')\big)}_{\text{action change at fixed $s$}}
\ +\
\underbrace{\mathrm{TV}\!\big(P_j(\cdot\mid s,a'),P_j(\cdot\mid s',a')\big)}_{\text{state change at fixed $a'$}}.
\]
By the definitions of $E^{\mathrm a}$ and $E^{\mathrm s}$,
\[
\mathrm{TV}\!\big(P_j(\cdot\mid s,a),P_j(\cdot\mid s,a')\big)
\ \le\ \sum_{k:\,a_k\neq a'_k} E^{\mathrm a}_{j\leftarrow k},
\qquad
\mathrm{TV}\!\big(P_j(\cdot\mid s,a'),P_j(\cdot\mid s',a')\big)
\ \le\ E^{\mathrm s}_{j\leftarrow i}.
\]
Therefore,
\[
\mathbb P[X_j\neq Y_j]
\ \le\ E^{\mathrm s}_{j\leftarrow i}\ +\ \sum_{k=1}^n E^{\mathrm a}_{j\leftarrow k}\,\mathbb P[A_k\neq A'_k]
\ \le\ E^{\mathrm s}_{j\leftarrow i}\ +\ \sum_{k=1}^n E^{\mathrm a}_{j\leftarrow k}\,\Pi(\pi)_{k\leftarrow i}.
\]

We now relate $|f(X)-f(Y)|$ to the coordinate-disagreement indicators. Define a sequence $Z_0=X,\ Z_1,\ldots,Z_n=Y$ by changing coordinates one at a time, with $Z_j$ obtained from $Z_{j-1}$ by replacing its $j$-th coordinate with $Y_j$. Then $Z_{j-1}$ and $Z_j$ agree off coordinate $j$, so by definition of $\delta_j(f)$,
\[
|f(Z_{j-1})-f(Z_j)|\ \le\ \delta_j(f)\,\mathbf 1\{X_j\neq Y_j\}.
\]
Summing along the chain and using the triangle inequality,
\[
|f(X)-f(Y)|\ \le\ \sum_{j=1}^n \delta_j(f)\,\mathbf 1\{X_j\neq Y_j\}.
\]
Taking expectations,
\[
|(T^\pi f)(s)-(T^\pi f)(s')|
\ \le\ \sum_{j=1}^n \delta_j(f)\,\mathbb P[X_j\neq Y_j]
\ \le\ \sum_{j=1}^n \delta_j(f)\,\Big(E^{\mathrm s}_{j\leftarrow i}+\sum_{k=1}^n E^{\mathrm a}_{j\leftarrow k}\,\Pi(\pi)_{k\leftarrow i}\Big).
\]
Taking the supremum over all $s,s'$ that agree off $i$ yields
\(
\delta_i(T^\pi f)\le \sum_j H^\pi_{j\leftarrow i}\,\delta_j(f)
\),
and the vector form follows.

The multi-step bound follows by induction. The base case $t=0$ is trivial. For the inductive step, applying the one-step bound to $g=(T^\pi)^t f$ gives $\delta(T^\pi g)\le (H^\pi)^\top\delta(g)\le (H^\pi)^\top((H^\pi)^\top)^t\delta(f)=((H^\pi)^\top)^{t+1}\delta(f)$, using the inductive hypothesis $\delta(g)\le((H^\pi)^\top)^t\delta(f)$ and entrywise nonnegativity of $(H^\pi)^\top$.

\subsection{Proof of Theorem~\protect\ref{thm:poisson}}

We note that the map $\pi\mapsto H^\pi=E^{\mathrm s}+E^{\mathrm a}\Pi(\pi)$ is continuous on the finite-state finite-action setting: $\Pi(\pi)$ is a finite maximum of total variations between $\pi$-marginals, which are continuous in $\pi$. Therefore $\rho(H^\pi)$ is upper semicontinuous on $\pi$, and the supremum $\lambda_\star=\sup_{\pi\in\Pi_{\mathrm{pol}}}\rho(H^\pi)$ is attained on the compact class $\Pi_{\mathrm{pol}}$.

By Lemma~\ref{lem:one-step}, for every $\pi\in\Pi_{\mathrm{pol}}$ and every bounded $f$,
\[
\delta(T^\pi f)\ \le\ (H^\pi)^\top \delta(f).
\]
Iterating gives
\[
\delta\big((T^\pi)^t f\big)\ \le\ \big((H^\pi)^\top\big)^t \delta(f),
\qquad t\ge 0.
\]

Fix any $\bar\lambda\in(\lambda_\star,1)$, where
\[
\lambda_\star=\sup_{\pi\in\Pi_{\mathrm{pol}}}\rho(H^\pi)<1.
\]
For each $\pi\in\Pi_{\mathrm{pol}}$, define
\[
w^\pi \coloneqq (\bar\lambda I-(H^\pi)^\top)^{-1}\mathbf 1 \in \mathbb R_{++}^n.
\]
This is well-defined because $\rho((H^\pi)^\top)=\rho(H^\pi)<\bar\lambda$. Moreover,
\[
(H^\pi)^\top w^\pi
=
\bar\lambda w^\pi-\mathbf 1
\ \le\
\bar\lambda w^\pi.
\]
Define the weighted sup norm on $\mathbb R^n$ by
\[
\|x\|_{w^\pi,\infty}\coloneqq \max_{i\in[n]} \frac{|x_i|}{w_i^\pi}.
\]
We claim that for every nonnegative vector $x\in\mathbb R_+^n$,
\[
\|(H^\pi)^\top x\|_{w^\pi,\infty}
\ \le\
\bar\lambda\,\|x\|_{w^\pi,\infty}.
\]
Indeed, for each $i$,
\begin{align*}
\big((H^\pi)^\top x\big)_i
&= \sum_j H^\pi_{j\leftarrow i}\,x_j
\ =\ \sum_j H^\pi_{j\leftarrow i}\,w_j^\pi\cdot\frac{x_j}{w_j^\pi}
\ \le\ \|x\|_{w^\pi,\infty}\sum_j H^\pi_{j\leftarrow i}\,w_j^\pi \\
&= \|x\|_{w^\pi,\infty}\,\big((H^\pi)^\top w^\pi\big)_i
\ \le\ \|x\|_{w^\pi,\infty}\,\bar\lambda\,w_i^\pi,
\end{align*}
where the last step uses $(H^\pi)^\top w^\pi\le\bar\lambda w^\pi$ from above. Dividing by $w_i^\pi$ and taking the max over $i$ gives the claim.
By induction,
\[
\big\|\big((H^\pi)^\top\big)^t x\big\|_{w^\pi,\infty}
\ \le\
\bar\lambda^t\,\|x\|_{w^\pi,\infty},
\qquad t\ge 0.
\]

Because the state and action spaces are finite, the map $\pi\mapsto H^\pi$ is continuous. Hence $\pi\mapsto w^\pi$ is continuous on the compact set $\Pi_{\mathrm{pol}}$, so
\[
m_{\bar\lambda}\coloneqq \inf_{\pi\in\Pi_{\mathrm{pol}}}\min_i w_i^\pi >0,
\qquad
M_{\bar\lambda}\coloneqq \sup_{\pi\in\Pi_{\mathrm{pol}}}\max_i w_i^\pi <\infty.
\]
Therefore, for every nonnegative vector $x$,
\[
\|x\|_{w^\pi,\infty}\le \frac{1}{m_{\bar\lambda}}\|x\|_\infty,
\qquad
\|x\|_\infty\le M_{\bar\lambda}\|x\|_{w^\pi,\infty},
\]
and thus
\[
\big\|\big((H^\pi)^\top\big)^t x\big\|_\infty
\ \le\
\frac{M_{\bar\lambda}}{m_{\bar\lambda}}\,
\bar\lambda^t\,
\|x\|_\infty.
\]
Applying this with $x=\delta(f)$ yields
\[
\|\delta((T^\pi)^t f)\|_\infty
\ \le\
C_{\bar\lambda,\Pi_{\mathrm{pol}}}\,\bar\lambda^t\,\|\delta(f)\|_\infty,
\qquad
C_{\bar\lambda,\Pi_{\mathrm{pol}}}\coloneqq \frac{M_{\bar\lambda}}{m_{\bar\lambda}}.
\]

Now assume in addition that, for each $\pi\in\Pi_{\mathrm{pol}}$, the Markov chain with kernel $P^\pi$ is irreducible, and let $d^\pi$ be its stationary distribution. Define
\[
g^\pi \coloneqq r^\pi-\bar r^\pi \mathbf 1,
\qquad
\bar r^\pi = \sum_{s\in\Sset} d^\pi(s)\,r^\pi(s).
\]
Since $\delta(g^\pi)=\delta(r^\pi)$, the bound above implies
\[
\sum_{t=0}^\infty \|\delta((T^\pi)^t g^\pi)\|_\infty < \infty.
\]

To handle the additive-constant ambiguity, work on the quotient space
\[
\mathcal B_0(\Sset)\coloneqq B(\Sset)\big/\mathrm{span}\{\mathbf 1\},
\]
and write $[f]$ for the equivalence class of $f$. Define
\[
\|[f]\|_\delta \coloneqq \|\delta(f)\|_\infty.
\]
Because $\delta(f)=0$ if and only if $f$ is constant, this is a well-defined norm on $\mathcal B_0(\Sset)$. The operator $T^\pi$ induces a linear map
\[
\widetilde T^\pi [f] \coloneqq [T^\pi f]
\]
satisfying
\[
\|(\widetilde T^\pi)^t [f]\|_\delta
\ \le\
C_{\bar\lambda,\Pi_{\mathrm{pol}}}\,\bar\lambda^t\,\|[f]\|_\delta.
\]
Hence the Neumann series converges in operator norm on $\mathcal B_0(\Sset)$, and we may define
\[
[h^\pi]
\ \coloneqq\
\sum_{t=0}^\infty (\widetilde T^\pi)^t [g^\pi].
\]
Then
\[
(I-\widetilde T^\pi)[h^\pi] = [g^\pi].
\]
Equivalently, for any representative $h^\pi$ of the class $[h^\pi]$, there exists a constant $c$ such that
\[
h^\pi - T^\pi h^\pi = g^\pi + c\mathbf 1.
\]
Applying the stationary distribution $d^\pi$ to both sides gives
\[
0 = d^\pi(g^\pi)+c = c,
\]
since $d^\pi(g^\pi)=0$ by definition of $\bar r^\pi$. Therefore
\[
h^\pi - T^\pi h^\pi = g^\pi = r^\pi-\bar r^\pi.
\]

Uniqueness up to an additive constant follows similarly: if $h-T^\pi h=0$, then
\[
(I-\widetilde T^\pi)[h]=0.
\]
Since $I-\widetilde T^\pi$ is invertible on the quotient space, $[h]=0$, so $h$ is constant.

Finally, using the triangle inequality for the oscillation seminorm and the multi-step bound,
\[
\delta(h^\pi)
\ \le\
\sum_{t=0}^\infty \delta((T^\pi)^t g^\pi)
\ \le\
\sum_{t=0}^\infty \big((H^\pi)^\top\big)^t \delta(g^\pi)
\ =\
\sum_{t=0}^\infty \big((H^\pi)^\top\big)^t \delta(r^\pi).
\]
Since $(H^\pi)^\top$ is entrywise nonnegative and $\rho(H^\pi)<1$, the Neumann series converges entrywise and
\[
\sum_{t=0}^\infty \big((H^\pi)^\top\big)^t
=
\big(I-(H^\pi)^\top\big)^{-1}.
\]
Hence
\[
\delta(h^\pi)
\ \le\
\big(I-(H^\pi)^\top\big)^{-1}\delta(r^\pi).
\]

This completes the proof.

\subsection{Spatial decay as a corollary of sparsity}\label{app:sparse}

Suppose there is an underlying undirected graph $G$ on $[n]$ such that the environment and the policy are local with respect to $G$ in the following sense:
\begin{itemize}
    \item $E^{\mathrm s}_{j\leftarrow i}=0$ unless $i$ lies in a fixed-radius neighborhood of $j$ in $G$;
    \item $E^{\mathrm a}_{j\leftarrow k}=0$ unless $k$ lies in a fixed-radius neighborhood of $j$ in $G$;
    \item $\Pi(\pi)_{k\leftarrow i}=0$ unless $i\in O_k$, and each observation scope $O_k$ is contained in a fixed-radius neighborhood of $k$ in $G$.
\end{itemize}
In general, the product $E^{\mathrm a}\Pi(\pi)$ need not have exactly the same sparsity pattern as $G$; rather, it induces a derived directed support graph
\[
G_H^\pi:\qquad i\to j \ \text{whenever}\ H^\pi_{j\leftarrow i}>0,
\qquad
H^\pi=E^{\mathrm s}+E^{\mathrm a}\Pi(\pi).
\]
This graph captures one-step closed-loop influence. Under the locality assumptions above, $G_H^\pi$ is a sparse finite-radius closure of the underlying graph.

For any pair of coordinates $i,j$, the entry
\[
\Big[\big((H^\pi)^\top\big)^t\Big]_{i j}
\]
can be nonzero only if there is a directed path of length at most $t$ from $i$ to $j$ in $G_H^\pi$, where edges are oriented as $i\to j$ whenever $H^\pi_{j\leftarrow i}>0$. Hence the Neumann-series bound
\[
\delta(h^\pi)\ \le\ \sum_{t=0}^\infty \big((H^\pi)^\top\big)^t\,\delta(r^\pi)
\]
shows that the contribution of coordinates outside a $\kappa$-hop neighborhood in $G_H^\pi$ is controlled by the tail
\[
\sum_{t=\kappa+1}^\infty \big((H^\pi)^\top\big)^t\,\delta(r^\pi),
\]
which decays exponentially whenever $\rho(H^\pi)<1$.

\subsection{Average-reward localized certificates and oracle truncation (synchronous)}

Under the sparsity conditions above, the truncated certificate
\[
\widehat\delta^{(\kappa)}=\sum_{t=0}^{\kappa}\big((H^\pi)^\top\big)^t\delta(r^\pi)
\]
from Theorem~\ref{thm:local-eval} is computable by local message passing on $\kappa$-neighborhoods in the support graph of $H^\pi$. The associated oracle truncated Poisson series $\widehat h_\kappa^\pi$ has bias modulo constants bounded by the same exponentially decaying Neumann tail. Using this oracle surrogate in a block KL-prox update yields the one-block improvement guarantee of Theorem~\ref{thm:block-improve}. Establishing a fully local policy-improvement algorithm requires an additional approximation or projection step beyond the present structural analysis and is left to future work.

\subsection{Asynchronous updates}\label{app:async}

For completeness we state the analogue of the main results when, at each step, all agents keep their states except for a randomly selected coordinate $J_t$ which is updated according to a site–selection distribution $\nu$ with full support. 

In this model the one–step operator is $K^\pi f(s)=\mathbb{E}_\pi[f(S_{t+1})\mid S_t=s]$ with
\[
\mathbb{P}[S_{t+1}=s^{(j\to y)} \mid S_t=s]
=
\nu_j \sum_{a} \Big(\prod_{k} \pi_k(a_k\mid s)\Big) P_j(y \mid s,a)
\quad\text{for }y\in\Sset_j.
\]
Define $E^{\mathrm{s}}$, $E^{\mathrm{a}}$, and $\Pi(\pi)$ as before and set
\[
M^\pi
=
(I-\mathrm{diag}\,\nu) + \mathrm{diag}\,\nu\, (E^{\mathrm{s}}+E^{\mathrm{a}}\Pi(\pi)).
\]
Then for all bounded $f$,
\[
\delta(K^\pi f)\le (M^\pi)^\top \delta(f).
\]
Consequently, if
\[
\lambda_\star^{\mathrm{async}}
\coloneqq
\sup_{\pi\in\Pi_{\mathrm{pol}}}\rho(M^\pi)<1,
\]
then for every $\bar\lambda\in(\lambda_\star^{\mathrm{async}},1)$ there exists a constant $C_{\bar\lambda}$ such that
\[
\|\delta((K^\pi)^t f)\|_\infty
\le
C_{\bar\lambda}\bar\lambda^t\|\delta(f)\|_\infty.
\]
Under the same irreducibility assumptions as in Theorem~\ref{thm:poisson}, the average-reward Poisson equation has a solution satisfying
\[
\delta(h^\pi)
\le
\big(I-(M^\pi)^\top\big)^{-1}\delta(r^\pi).
\]

The entrywise bound follows by conditioning on the selected coordinate $J\sim \nu$. Fix $i$ and $s,s'$ that agree off coordinate $i$, and couple the actions $A,A'$ exactly as in the proof of Lemma~\ref{lem:one-step}. If $J=i$, then $S_{t+1}$ and $S'_{t+1}$ inherit all coordinates other than $i$ from $s,s'$, which already agree, so they differ only at coordinate $i$, with $\mathbb P[S_{t+1,i}\ne S'_{t+1,i}]\le H^\pi_{i\leftarrow i}$ by the same maximal-coupling argument used in Lemma~\ref{lem:one-step}, and $|f(S_{t+1})-f(S'_{t+1})|\le \delta_i(f)\,H^\pi_{i\leftarrow i}$ in expectation. If $J=j\ne i$, then coordinate $i$ is not updated, so $S_{t+1,i}=s_i\ne s'_i=S'_{t+1,i}$, contributing $\delta_i(f)$; and coordinate $j$ is updated, contributing $H^\pi_{j\leftarrow i}\,\delta_j(f)$ in expectation. Averaging over $J\sim\nu$ gives
\[
\delta_i(K^\pi f)\le (1-\nu_i)\,\delta_i(f) + \nu_i H^\pi_{i\leftarrow i}\,\delta_i(f) + \sum_{j\ne i}\nu_j H^\pi_{j\leftarrow i}\,\delta_j(f) = \sum_{j=1}^n M^\pi_{j\leftarrow i}\,\delta_j(f),
\]
which is the entrywise form of $\delta(K^\pi f)\le (M^\pi)^\top \delta(f)$. The remaining steps (Poisson decay, exponential bound on iterates) repeat the proof of Theorem~\ref{thm:poisson} with $H^\pi$ replaced by $M^\pi$.

\subsection{Extension of Theorem~\protect\ref{thm:poisson} to the weighted case}

The same argument also yields weighted oscillation bounds. We record the statement because it is often useful for certifying spectral-radius conditions through weighted norms.

Fix $w\in\mathbb R_{++}^n$ and $W=\mathrm{diag}(w)$. For any bounded $f$ and $i\in[n]$,
\[
\delta_i^w(T^\pi f)\ =\ w_i\,\delta_i(T^\pi f)\ \le\ w_i \sum_{j=1}^n H^\pi_{j\leftarrow i}\, \delta_j(f)\ =\ \sum_{j=1}^n \big(W^{-1}H^\pi W\big)_{j\leftarrow i}\, \delta_j^w(f).
\]
This is the entrywise weighted one--step contraction. Iterating gives the multi--step bound. For the Poisson resolvent, one repeats the quotient-space proof of Theorem~\ref{thm:poisson} with the seminorm $\|f\|_{\delta^w}=\|\delta^w(f)\|_\infty$. Let
\[
\widetilde H^{\pi,w}\coloneqq W^{-1}H^\pi W.
\]
With our convention $H^\pi_{j\leftarrow i}=H^\pi[j,i]$, the entrywise bound above becomes the vector inequality $\delta^w(T^\pi f)\le (\widetilde H^{\pi,w})^\top \delta^w(f)$, with a transpose for the same reason as in the unweighted Lemma~\ref{lem:one-step}. The Neumann series $\sum_{t\ge 0}((\widetilde H^{\pi,w})^\top)^t$ converges entrywise when $\rho(\widetilde H^{\pi,w})<1$ (noting $\rho((\widetilde H^{\pi,w})^\top)=\rho(\widetilde H^{\pi,w})$), and
\[
\delta^w(h^\pi)\ \le\ \sum_{t=0}^\infty \big((\widetilde H^{\pi,w})^\top\big)^t \delta^w(r^\pi)\ =\ \big(I-(\widetilde H^{\pi,w})^\top\big)^{-1}\delta^w(r^\pi).
\]
Since $\widetilde H^{\pi,w}=W^{-1}H^\pi W$ is similar to $H^\pi$, the spectral radius is unchanged: $\rho(\widetilde H^{\pi,w})=\rho(H^\pi)$. The weighted version is therefore useful for sharpening operator-norm certificates, not for changing the spectral condition itself. The truncated series bound and the uniform power bound over a compact policy class are identical to the unweighted case after replacing $H^\pi$ by $\widetilde H^{\pi,w}$.

\section{Proofs from Section~\protect\ref{sec:softmax}}
\label{app:softmax-Pi}

We bound the policy sensitivity matrix $\Pi(\pi)$ for entropy-regularized (temperature-$\tau$) softmax policies in terms of per-agent logit Lipschitz constants. Throughout, total variation is normalized: $\TV(\mu,\nu)=\tfrac12\|\mu-\nu\|_1$.

\begin{lemma}[Softmax Lipschitz constant in total variation]
\label{lem:softmax-lip}
Fix $m\in\mathbb N$ and $\tau>0$. For $u,v\in\mathbb R^m$, define
\[
\mathrm{soft}_\tau(u)_a \ \coloneqq\ \frac{\exp(u_a/\tau)}{\sum_{b=1}^m \exp(u_b/\tau)}\,,\qquad a\in[m].
\]
Then
\[
\TV\big(\mathrm{soft}_\tau(u),\ \mathrm{soft}_\tau(v)\big)\ \le\ \frac{1}{2\tau}\,\|u-v\|_\infty.
\]
The constant $1/(2\tau)$ is optimal: for $m\ge 2$, the supremum of $\TV(\mathrm{soft}_\tau(u),\mathrm{soft}_\tau(v))/\|u-v\|_\infty$ over $u\ne v$ equals $1/(2\tau)$, approached (but not attained) as $\|u-v\|_\infty\to 0$ along a balanced $\pm1$ direction starting from a uniform softmax. Sharpness in this asymptotic sense is exactly what makes $1/(2\tau)$ the best uniform Lipschitz constant.
\end{lemma}

\begin{proof}
Write $p=\mathrm{soft}_\tau(u)$ and $q=\mathrm{soft}_\tau(v)$. By the mean value theorem on the line segment $w(t)=v+t(u-v)$, $t\in[0,1]$,
\[
\|p-q\|_1 \ =\ \left\| \int_0^1 \frac{\mathrm d}{\mathrm dt}\mathrm{soft}_\tau(w(t))\,\mathrm dt \right\|_1
\ \le\ \int_0^1 \left\| D\mathrm{soft}_\tau(w(t))\,(u-v)\right\|_1 \mathrm dt,
\]
where $D\mathrm{soft}_\tau(w)$ is the Jacobian. For $w\in\mathbb R^m$ with $r=\mathrm{soft}_\tau(w)$,
\[
D\mathrm{soft}_\tau(w)\ =\ \frac{1}{\tau}\,\big(\mathrm{diag}(r)-r r^\top\big).
\]
Hence, using the induced operator norm from $\ell_\infty$ to $\ell_1$,
\[
\|p-q\|_1 \ \le\ \frac{1}{\tau}\,\left(\sup_{t\in[0,1]}\big\|\mathrm{diag}(r(t))-r(t)r(t)^\top\big\|_{\infty\to 1}\right)\,\|u-v\|_\infty,\qquad r(t)=\mathrm{soft}_\tau(w(t)).
\]
We claim that for every probability vector $r$,
\[
\big\|\mathrm{diag}(r)-r r^\top\big\|_{\infty\to 1}\ \le\ 1. \qquad (\star)
\]
This yields $\|p-q\|_1\le \frac{1}{\tau}\|u-v\|_\infty$ and therefore
$\TV(p,q)\le \frac{1}{2\tau}\|u-v\|_\infty$.

It remains to prove $(\star)$. Let $J(r)=\mathrm{diag}(r)-r r^\top$ and fix $\delta\in\mathbb R^m$ with $\|\delta\|_\infty\le 1$. Then
\[
(J(r)\delta)_i \ =\ r_i\big(\delta_i - \langle r,\delta\rangle\big),\qquad i\in[m],
\]
hence
\[
\|J(r)\delta\|_1 \ =\ \sum_{i=1}^m r_i\,\big|\delta_i-\mu\big|\ =\ \mathbb E_{I\sim r}\big[\,|\delta_I-\mu|\,\big],\qquad \mu\coloneqq \langle r,\delta\rangle \in [-1,1].
\]
Let $X=\delta_I$ with $I\sim r$; then $X\in[-1,1]$ and $\mathbb E[X]=\mu$. By Jensen's inequality and the variance bound for bounded random variables,
\[
\mathbb E\,|X-\mu|\ \le\ \sqrt{\mathrm{Var}(X)}\ \le\ \sqrt{1-\mu^2}\ \le\ 1,
\]
where the middle step uses $\mathrm{Var}(X)\le \mathbb E[X^2]-\mu^2\le 1-\mu^2$ since $X^2\le 1$. Therefore $\|J(r)\delta\|_1\le 1$ for all $\delta$ with $\|\delta\|_\infty\le 1$, proving $(\star)$. The bound is tight in the limit $\mu\to 0$ with $X$ taking values $\pm 1$ equally.
\end{proof}

We now translate Lemma~\ref{lem:softmax-lip} into a bound on the policy sensitivity matrix $\Pi(\pi)$ for product-form, temperature-$\tau$ softmax policies with local logits.

\begin{definition}[Local logits and per-coordinate logit Lipschitz constants]
For each agent $k$, suppose there is a logit function $g_k:\Sset_{O_k}\times \Aset_k\to\mathbb R$ such that
\[
\pi_k(a_k\mid s_{O_k})\ \propto\ \exp\!\big(g_k(s_{O_k},a_k)/\tau\big),\qquad a_k\in\Aset_k,
\]
with temperature $\tau>0$. For $i\in[n]$, define the one-coordinate logit Lipschitz constant
\[
L_{k\leftarrow i}\ \coloneqq\ \sup_{s_{-i}=s'_{-i}}\ \big\| g_k(s_{O_k},\cdot)\ -\ g_k(s'_{O_k},\cdot)\big\|_\infty,
\]
where the sup is over $s,s'\in \Sset$ that differ only on coordinate $i$.
\end{definition}

\begin{lemma}[Softmax temperature controls $\Pi(\pi)$]
Under the setup above, for all $k,i\in[n]$,
\[
\Pi_{k\leftarrow i}(\pi)\ \le\ \min\!\left\{\,1,\ \frac{L_{k\leftarrow i}}{2\tau}\,\right\}.
\]
In particular, if $i\notin O_k$ then $L_{k\leftarrow i}=0$ and $\Pi_{k\leftarrow i}(\pi)=0$.
\end{lemma}

\begin{proof}
Fix $k$ and $i$. If $i\notin O_k$ then $g_k(s_{O_k},\cdot)$ is unchanged when $s_i$ varies, hence $L_{k\leftarrow i}=0$ and $\pi_k(\cdot\mid s_{O_k})=\pi_k(\cdot\mid s'_{O_k})$ for all $s_{-i}=s'_{-i}$, giving $\Pi_{k\leftarrow i}(\pi)=0$.

Assume $i\in O_k$. For $s,s'$ with $s_{-i}=s'_{-i}$, apply Lemma~\ref{lem:softmax-lip} with $u=g_k(s_{O_k},\cdot)$ and $v=g_k(s'_{O_k},\cdot)$ to obtain
\[
\TV\big(\pi_k(\cdot\mid s_{O_k}),\ \pi_k(\cdot\mid s'_{O_k})\big)\ \le\ \frac{1}{2\tau}\,\big\|g_k(s_{O_k},\cdot)-g_k(s'_{O_k},\cdot)\big\|_\infty\ \le\ \frac{L_{k\leftarrow i}}{2\tau}.
\]
Taking the supremum over such $s,s'$ yields $\Pi_{k\leftarrow i}(\pi)\le L_{k\leftarrow i}/(2\tau)$. The bound is trivially capped by $1$ because total variation lies in $[0,1]$.
\end{proof}

\begin{remark}[Sharpness]
The constant $1/(2\tau)$ inherited from Lemma~\ref{lem:softmax-lip} is optimal in the asymptotic sense described there: it cannot be improved as $\|g_k(s_{O_k},\cdot)-g_k(s'_{O_k},\cdot)\|_\infty\to 0$. Consequently, controlling $\Pi(\pi)$ uniformly over a policy class amounts to lower-bounding the entropy temperature $\tau$ and upper-bounding the one-coordinate logit oscillations $L_{k\leftarrow i}$.
\end{remark}

\section{Proofs from Section~\protect\ref{sec:algo}}
\label{app:algo}

All state and action spaces are finite. Total variation is normalized as $\TV(\mu,\nu)=\tfrac12\sum_x|\mu(x)-\nu(x)|$. Coordinatewise oscillations are $\delta_i(f)=\sup\{|f(x)-f(y)|: x_{-i}=y_{-i}\}$ and $\delta(f)=(\delta_i(f))_{i=1}^n$. We use Theorem~\ref{thm:poisson} from the main text.

\subsection{Proof of Theorem~\ref{thm:local-eval}}

Fix a policy $\pi$. Let
\[
g^\pi \coloneqq r^\pi-\bar r^\pi,
\qquad
b^\pi \coloneqq \delta(r^\pi)=\delta(g^\pi),
\qquad
H^\pi \coloneqq E^{\mathrm s}+E^{\mathrm a}\Pi(\pi).
\]
By Theorem~\ref{thm:poisson},
\[
\delta(h^\pi)\ \le\ \sum_{t=0}^{\infty}\big((H^\pi)^\top\big)^t\,b^\pi.
\]
Define
\[
\widehat\delta^{(\kappa)}
\coloneqq
\sum_{t=0}^{\kappa}\big((H^\pi)^\top\big)^t\,b^\pi,
\qquad
R^{(\kappa)}
\coloneqq
\sum_{t=\kappa+1}^{\infty}\big((H^\pi)^\top\big)^t\,b^\pi.
\]
Then immediately
\[
\delta(h^\pi)\ \le\ \widehat\delta^{(\kappa)}+R^{(\kappa)},
\]
which proves the certificate statement.

For the tail bound,
\[
\|R^{(\kappa)}\|_\infty
\ \le\
\sum_{t=\kappa+1}^{\infty}
\big\|\big((H^\pi)^\top\big)^t\big\|_{\infty\to\infty}\,\|b^\pi\|_\infty
\ \le\
\sum_{t=\kappa+1}^{\infty}
C\lambda^t\,\|b^\pi\|_\infty
\ =\
\frac{C}{1-\lambda}\lambda^{\kappa+1}\|b^\pi\|_\infty.
\]

For locality, let $G_H^\pi$ be the support graph of $H^\pi$ (edge $i\to j$ when $H^\pi_{j\leftarrow i}>0$). Expanding the matrix power,
\[
\Big[\big((H^\pi)^\top\big)^t b^\pi\Big]_i
\ =\
\sum_{i=k_0,k_1,\dots,k_t}
H^\pi_{k_1\leftarrow k_0}\,H^\pi_{k_2\leftarrow k_1}\,\cdots\,H^\pi_{k_t\leftarrow k_{t-1}}\,b^\pi_{k_t},
\]
which sums over directed length-$t$ paths $k_0=i\to k_1\to\cdots\to k_t$ in $G_H^\pi$. The contribution is nonzero only when every $H^\pi_{k_{\ell+1}\leftarrow k_\ell}$ is positive, i.e.\ when $k_t$ is reachable from $i$ in $t$ steps. Hence the truncated sum $\widehat\delta_i^{(\kappa)}$ depends only on entries of $H^\pi$ and $b^\pi$ supported on the $\kappa$-hop \emph{out-ball} of $i$ in $G_H^\pi$ (agents reachable from $i$ in at most $\kappa$ directed edges), and it can be computed by $\kappa$ rounds of message passing along the reverse edges of that ball.

For the oracle truncation bias, define
\[
\widehat h_\kappa^\pi \coloneqq \sum_{t=0}^{\kappa}(T^\pi)^t g^\pi.
\]
From the quotient-space construction used in the proof of Theorem~\ref{thm:poisson},
\[
[h^\pi]
=
\sum_{t=0}^\infty (\widetilde T^\pi)^t [g^\pi],
\qquad
[\widehat h_\kappa^\pi]
=
\sum_{t=0}^{\kappa} (\widetilde T^\pi)^t [g^\pi],
\]
where $\widetilde T^\pi$ denotes the induced operator on $B(\Sset)/\mathrm{span}\{\mathbf 1\}$. Therefore
\[
[\widehat h_\kappa^\pi-h^\pi]
=
-\sum_{t=\kappa+1}^{\infty}(\widetilde T^\pi)^t[g^\pi].
\]
Applying the oscillation bound gives
\[
\delta(\widehat h_\kappa^\pi-h^\pi)
\ \le\
\sum_{t=\kappa+1}^{\infty}\big((H^\pi)^\top\big)^t\,b^\pi
=
R^{(\kappa)}.
\]

Now use the standard identity
\[
\inf_{c\in\mathbb R}\|f-c\mathbf 1\|_\infty
=
\frac12\,\mathrm{osc}(f),
\qquad
\mathrm{osc}(f)\le \sum_{i=1}^n \delta_i(f)=\|\delta(f)\|_1.
\]
Applying this to $f=\widehat h_\kappa^\pi-h^\pi$ yields
\[
B_\kappa^\pi
=
\inf_{c\in\mathbb R}\|\widehat h_\kappa^\pi-h^\pi-c\mathbf 1\|_\infty
\ \le\
\frac12\,\|\delta(\widehat h_\kappa^\pi-h^\pi)\|_1
\ \le\
\frac12\sum_{i=1}^n R_i^{(\kappa)}.
\]
Finally,
\[
\sum_{i=1}^n R_i^{(\kappa)}
\le
n\|R^{(\kappa)}\|_\infty
\le
\frac{nC}{1-\lambda}\lambda^{\kappa+1}\|b^\pi\|_\infty,
\]
which gives
\[
B_\kappa^\pi
\le
\frac{nC}{2(1-\lambda)}\lambda^{\kappa+1}\|b^\pi\|_\infty.
\]
This completes the proof.

\subsection{Preliminaries for the block-improvement step}

Throughout this subsection, fix a baseline product policy $\pi$ and write
\[
A^\pi(s,a)
=
r(s,a)-\bar r^\pi+\sum_y P(y\mid s,a)h^\pi(y)-h^\pi(s).
\]

\begin{lemma}[Advantage perturbation via value perturbation modulo constants]
\label{lem:adv-perturb}
Let $\widehat h$ be any function on $\Sset$ and define
\[
\widehat A(s,a)
=
r(s,a)-\bar r^\pi+\sum_y P(y\mid s,a)\widehat h(y)-\widehat h(s).
\]
Then
\[
\sup_{s\in\Sset,\ a\in\Aset} |\widehat A(s,a)-A^\pi(s,a)|
\ \le\
2\inf_{c\in\mathbb R}\|\widehat h-h^\pi-c\mathbf 1\|_\infty.
\]
\end{lemma}

\begin{proof}
For any constant $c\in\mathbb R$, replacing $\widehat h$ by $\widehat h+c\mathbf 1$ does not change $\widehat A$, because the additive constant cancels between the transition term and the state term. Hence, for every $c$,
\begin{align*}
|\widehat A(s,a)-A^\pi(s,a)|
&=
\left|
\sum_y P(y\mid s,a)\big(\widehat h(y)+c-h^\pi(y)\big)
-
\big(\widehat h(s)+c-h^\pi(s)\big)
\right| \\
&\le
\sum_y P(y\mid s,a)\,|\widehat h(y)+c-h^\pi(y)|
+
|\widehat h(s)+c-h^\pi(s)| \\
&\le
2\|\widehat h-h^\pi-c\mathbf 1\|_\infty.
\end{align*}
Taking the supremum over $(s,a)$ and then the infimum over $c$ proves the claim.
\end{proof}

\begin{lemma}[Per-state KL-prox duality]
\label{lem:prox-duality}
Fix a finite action set $\mathcal A$, a reference distribution $q\in\Delta(\mathcal A)$ with $q(a)>0$ for all $a$ (or, equivalently, restrict the maximization to distributions absolutely continuous with respect to $q$), a score vector $g\in\mathbb R^{\mathcal A}$, and $\eta>0$. Then
\[
\max_{p\in\Delta(\mathcal A)}
\Big\{
\langle p,g\rangle-\eta\,\KL(p\|q)
\Big\}
=
\eta\log\sum_{a\in\mathcal A} q(a)e^{g(a)/\eta},
\]
attained uniquely at
\[
p^\star(a)\propto q(a)e^{g(a)/\eta}.
\]
Moreover,
\[
\langle p^\star-q,\ g\rangle \ \ge\ \eta\,\KL(p^\star\|q).
\]
\end{lemma}

\begin{proof}
The maximization is the standard log-sum-exp / Donsker--Varadhan variational formula. The Lagrangian for the constraint $\sum_a p(a)=1$ has gradient $g(a)-\eta\log(p(a)/q(a))-\eta-\lambda=0$, which gives $p(a)\propto q(a)e^{g(a)/\eta}$. The value of the objective at $p^\star$ is $\eta\log\sum_a q(a)e^{g(a)/\eta}$ by direct substitution, and uniqueness follows from strict convexity of $\KL(\cdot\|q)$ on the simplex.

For the final inequality, evaluate the objective at $p=q$: $\langle q,g\rangle-\eta\KL(q\|q)=\langle q,g\rangle$. By optimality of $p^\star$,
\[
\langle p^\star,g\rangle-\eta\KL(p^\star\|q)\ \ge\ \langle q,g\rangle,
\]
which rearranges to $\langle p^\star-q,g\rangle\ge \eta\KL(p^\star\|q)$.
\end{proof}

\begin{lemma}[One-block performance difference]
\label{lem:block-pdl}
Let $\pi$ be a product policy and let $\mu$ be another product policy such that $\mu_{-k}=\pi_{-k}$ for some agent $k$. Assume the Markov chain under $\mu$ is irreducible with stationary distribution $d^\mu$. Define
\[
g_{k,\star}^\pi(s,a_k)
\coloneqq
\E_{a_{-k}\sim \prod_{j\ne k}\pi_j(\cdot\mid s)}
\big[A^\pi(s,(a_k,a_{-k}))\big].
\]
Then
\[
\bar r(\mu)-\bar r(\pi)
=
\E_{S\sim d^\mu}
\Big[
\langle
\mu_k(\cdot\mid S)-\pi_k(\cdot\mid S),\
g_{k,\star}^\pi(S,\cdot)
\rangle
\Big].
\]
\end{lemma}

\begin{proof}
Using the Poisson equation for $h^\pi$,
\[
h^\pi-T^\pi h^\pi=r^\pi-\bar r^\pi,
\]
one obtains the standard average-reward performance-difference identity
\[
\bar r(\mu)-\bar r(\pi)
=
\E_{S\sim d^\mu,\ A\sim \mu(\cdot\mid S)}
\big[A^\pi(S,A)\big].
\]
Indeed,
\begin{align*}
\E_{d^\mu,\mu}[A^\pi]
&=
\E_{d^\mu,\mu}\!\left[
r(S,A)-\bar r^\pi+\sum_y P(y\mid S,A)h^\pi(y)-h^\pi(S)
\right] \\
&=
\E_{d^\mu,\mu}[r(S,A)]-\bar r^\pi
+
\E_{S'\sim d^\mu}[h^\pi(S')]
-
\E_{S\sim d^\mu}[h^\pi(S)] \\
&=
\bar r(\mu)-\bar r(\pi),
\end{align*}
where the middle two terms cancel by stationarity of $d^\mu$ under $\mu$.

Now use $\mu_{-k}=\pi_{-k}$:
\[
\E_{A\sim \mu(\cdot\mid s)}[A^\pi(s,A)]
=
\sum_{a_k}\mu_k(a_k\mid s)\,g_{k,\star}^\pi(s,a_k)
=
\langle \mu_k(\cdot\mid s), g_{k,\star}^\pi(s,\cdot)\rangle.
\]
Likewise,
\[
\langle \pi_k(\cdot\mid s), g_{k,\star}^\pi(s,\cdot)\rangle
=
\E_{A\sim \pi(\cdot\mid s)}[A^\pi(s,A)]
=
0.
\]
Therefore
\[
\E_{A\sim \mu(\cdot\mid s)}[A^\pi(s,A)]
=
\langle
\mu_k(\cdot\mid s)-\pi_k(\cdot\mid s),\,
g_{k,\star}^\pi(s,\cdot)
\rangle.
\]
Averaging over $S\sim d^\mu$ proves the claim.
\end{proof}

\subsection{Proof of Theorem~\ref{thm:block-improve}}

Fix $\pi$, $k$, and $\kappa$. First, by Lemma~\ref{lem:adv-perturb} with $\widehat h=\widehat h_\kappa^\pi$,
\[
\sup_{s,a}
|\widehat A_\kappa^\pi(s,a)-A^\pi(s,a)|
\le
2 B_\kappa^\pi.
\]
Taking expectation over $a_{-k}\sim \prod_{j\ne k}\pi_j(\cdot\mid s)$ preserves the sup norm, so
\[
\|\widehat g_{k,\kappa}^\pi-g_{k,\star}^\pi\|_\infty
\le
2 B_\kappa^\pi.
\]

Now consider the policy $\mu$ defined by
\[
\mu_k(\cdot\mid s)
\propto
\pi_k(\cdot\mid s)\exp\!\big(\widehat g_{k,\kappa}^\pi(s,\cdot)/\eta\big),
\qquad
\mu_{-k}=\pi_{-k}.
\]
For each fixed state $s$, Lemma~\ref{lem:prox-duality} with
\[
q=\pi_k(\cdot\mid s),
\qquad
g=\widehat g_{k,\kappa}^\pi(s,\cdot),
\qquad
p^\star=\mu_k(\cdot\mid s)
\]
gives
\[
\langle
\mu_k(\cdot\mid s)-\pi_k(\cdot\mid s),\
\widehat g_{k,\kappa}^\pi(s,\cdot)
\rangle
\ge
\eta\,\KL\big(\mu_k(\cdot\mid s)\,\|\,\pi_k(\cdot\mid s)\big).
\]
Subtract and add the exact logit:
\begin{align*}
\langle
\mu_k-\pi_k,\
g_{k,\star}^\pi
\rangle
&=
\langle
\mu_k-\pi_k,\
\widehat g_{k,\kappa}^\pi
\rangle
+
\langle
\mu_k-\pi_k,\
g_{k,\star}^\pi-\widehat g_{k,\kappa}^\pi
\rangle \\
&\ge
\eta\,\KL(\mu_k\|\pi_k)
-
\|\mu_k-\pi_k\|_1
\,
\|\widehat g_{k,\kappa}^\pi-g_{k,\star}^\pi\|_\infty \\
&\ge
\eta\,\KL(\mu_k\|\pi_k)
-
2\|\widehat g_{k,\kappa}^\pi-g_{k,\star}^\pi\|_\infty \\
&\ge
\eta\,\KL(\mu_k\|\pi_k)
-
4 B_\kappa^\pi,
\end{align*}
where all distributions are evaluated at the same state $s$ and we used $\|\mu_k-\pi_k\|_1\le 2$.

Finally, average over $S\sim d^\mu$ and apply Lemma~\ref{lem:block-pdl}:
\[
\bar r(\mu)-\bar r(\pi)
=
\E_{S\sim d^\mu}
\Big[
\langle
\mu_k(\cdot\mid S)-\pi_k(\cdot\mid S),\
g_{k,\star}^\pi(S,\cdot)
\rangle
\Big]
\]
\[
\ge
\eta\,\E_{S\sim d^\mu}
\Big[
\KL\big(\mu_k(\cdot\mid S)\,\|\,\pi_k(\cdot\mid S)\big)
\Big]
-
4 B_\kappa^\pi.
\]
Substituting the bound from Theorem~\ref{thm:local-eval},
\[
B_\kappa^\pi
\le
\frac{nC}{2(1-\lambda)}\lambda^{\kappa+1}\|b^\pi\|_\infty,
\]
yields
\[
\bar r(\mu)-\bar r(\pi)
\ge
\eta\,\E_{S\sim d^\mu}
\Big[
\KL\big(\mu_k(\cdot\mid S)\,\|\,\pi_k(\cdot\mid S)\big)
\Big]
-
\frac{2nC}{1-\lambda}\lambda^{\kappa+1}\|b^\pi\|_\infty.
\]
This proves the theorem.

\section{Miscellaneous Results}

\subsection{Example: A Coupled System where Policy Smoothing is Ineffective}
\label{app:policy-ex}

Our framework also captures the failure mode where policy smoothing cannot create locality because the environment itself has strong direct state-to-state feedback. Consider two agents with binary states and arbitrary actions. Suppose the next states deterministically copy each other:
\[
P_1(s'_1=1\mid s,a)=\mathbf 1\{s_2=1\},
\qquad
P_2(s'_2=1\mid s,a)=\mathbf 1\{s_1=1\}.
\]
Actions have no effect on the next state. Hence
\[
E^{\mathrm a}=0,
\qquad
E^{\mathrm s}
=
\begin{pmatrix}
0 & 1\\
1 & 0
\end{pmatrix},
\]
where rows and columns are indexed by the influence convention $j\leftarrow i$. Therefore
\[
H^\pi=E^{\mathrm s}+E^{\mathrm a}\Pi(\pi)=E^{\mathrm s},
\qquad
\rho(H^\pi)=1.
\]
The policy sensitivity matrix $\Pi(\pi)$ is irrelevant because the action-coupling channel is absent. Thus no amount of policy smoothing can make $\rho(H^\pi)<1$ in this example. This illustrates the complementary failure mode to the policy-induced locality examples: if the environment has an unstable direct state-feedback loop, the policy cannot remove it unless actions actually mediate that coupling.

\subsection{Example: policy-induced locality in a nonlocal-looking system.}\label{app:nonlocal-ex}
Consider a hub–and–spoke system with $n\ge3$ agents. 

Let $P_1$ be constant and, for $j>1$,
$P_j(s'_j=1\mid s,a)=\mathbf{1}\{a_1=1\}$. Then $E^{\mathrm s}=0$ and
$E^{\mathrm a}_{j\leftarrow 1}=1$ for all $j>1$ (others zero). A policy-independent,
action-supremum certificate declares the system non-local ($\|C\|_\infty=1$) \citep{qu_scalabale_marl_2020}.

Under a temperature-$\tau$ softmax for agent~1 with logit Lipschitz constants $\{L_{1\leftarrow i}\}$,
Lemma~\ref{lem:Pi-softmax} gives $\Pi_{1\leftarrow i}\le L_{1\leftarrow i}/(2\tau)$ and hence
\[
\rho(H^\pi)\ \le\ \rho\!\Big(E^{\mathrm a}\frac{L}{2\tau}\Big)
\ =\ \frac{1}{2\tau}\sum_{i>1}L_{1\leftarrow i}
\ \le\ \frac{(n-1)L_{\max}}{2\tau},
\]
Thus locality is certified whenever $\tau>(n-1)L_{\max}/2$, and if the hub ignores its inputs ($L\equiv0$) then $\rho(H^\pi)=0$.
This illustrates how locality can be \emph{policy-induced}, while worst-case, policy-independent tests cannot detect it.

\subsection{Properties of the Coordinatewise Oscillation Seminorm}
\label{app:oscillation_proof}

\newcommand{\norm}[1]{\left\lVert #1 \right\rVert}

Let $V = B(\Sset)$ be the vector space of all bounded real-valued functions on the finite state space $\Sset = \prod_{i=1}^n \Sset_i$.

\begin{definition}[Coordinatewise Oscillation]
For a function $f \in V$ and a coordinate $i \in [n]$, the $i$-oscillation is:
\[
\delta_i(f) = \sup\Big\{\,|f(x)-f(y)|: x,y\in\Sset,\ x_{-i}=y_{-i}\,\Big\}.
\]
We define the function $p: V \to \R$ as the maximum oscillation:
\[
p(f) = \norm{\delta(f)}_\infty = \max_{i \in [n]} \delta_i(f).
\]
\end{definition}

\begin{proposition}
The function $p(f) = \norm{\delta(f)}_\infty$ is a seminorm on the vector space $V$.
\end{proposition}

\begin{proof}
To prove that $p(f)$ is a seminorm, we must verify three properties:
\begin{enumerate}
    \item \textbf{Non-negativity:} $p(f) \ge 0$ for all $f \in V$.
    \item \textbf{Absolute Homogeneity:} $p(c f) = |c| p(f)$ for all $f \in V$ and scalar $c \in \R$.
    \item \textbf{Subadditivity (Triangle Inequality):} $p(f + g) \le p(f) + p(g)$ for all $f, g \in V$.
\end{enumerate}

\textbf{1. Non-negativity:}
The absolute value $|f(x)-f(y)|$ is always non-negative. The supremum of a set of non-negative numbers, $\delta_i(f)$, is also non-negative. The maximum of a set of non-negative numbers, $p(f)$, is therefore non-negative.

\textbf{2. Absolute Homogeneity:}
For any $f \in V$ and $c \in \R$:
\begin{align*}
\delta_i(c f) &= \sup_{x_{-i}=y_{-i}} |(c f)(x) - (c f)(y)| \\
&= \sup_{x_{-i}=y_{-i}} |c \cdot (f(x) - f(y))| \\
&= |c| \cdot \sup_{x_{-i}=y_{-i}} |f(x) - f(y)| = |c| \cdot \delta_i(f).
\end{align*}
Taking the maximum over all $i$:
\[
p(c f) = \max_{i \in [n]} \delta_i(c f) = \max_{i \in [n]} \big( |c| \cdot \delta_i(f) \big) = |c| \cdot \max_{i \in [n]} \delta_i(f) = |c| \cdot p(f).
\]

\textbf{3. Subadditivity:}
For any $f, g \in V$:
\begin{align*}
\delta_i(f + g) &= \sup_{x_{-i}=y_{-i}} |(f+g)(x) - (f+g)(y)| \\
&= \sup_{x_{-i}=y_{-i}} |(f(x) - f(y)) + (g(x) - g(y))| \\
&\le \sup_{x_{-i}=y_{-i}} \big( |f(x) - f(y)| + |g(x) - g(y)| \big) \quad \text{(by the triangle inequality for $\R$)} \\
&\le \sup_{x_{-i}=y_{-i}} |f(x) - f(y)| + \sup_{x_{-i}=y_{-i}} |g(x) - g(y)| \quad \text{(by a standard property of suprema)} \\
&= \delta_i(f) + \delta_i(g).
\end{align*}
Now, taking the maximum over all $i$:
\[
p(f + g) = \max_{i \in [n]} \delta_i(f + g) \le \max_{i \in [n]} (\delta_i(f) + \delta_i(g)).
\]
For any $i$, we know $\delta_i(f) \le \max_j \delta_j(f) = p(f)$ and $\delta_i(g) \le \max_j \delta_j(g) = p(g)$. Thus:
\[
p(f + g) \le \max_{i \in [n]} (p(f) + p(g)) = p(f) + p(g).
\]
This completes the proof.
\end{proof}

\begin{remark}[Why it is a seminorm, not a norm]
A norm requires $p(f) = 0 \iff f = 0$. For our $p(f)$, if $f(x) = c$ for some non-zero constant $c$, then $f \ne 0$. However, for any $i$ and any pair $x,y$ with $x_{-i}=y_{-i}$, $f(x) = c$ and $f(y) = c$, so $|f(x)-f(y)| = 0$. This implies $\delta_i(f) = 0$ for all $i$, and thus $p(f) = 0$. Since $p(f) = 0$ for non-zero constant functions, $p(f)$ is a seminorm. In fact, $p(f) = 0 \iff f \text{ is a constant function.}$
\end{remark}

\begin{proposition}
The value $p(f) = \norm{\delta(f)}_\infty$ is the (best) Lipschitz constant of $f$ with respect to the Hamming distance $d_H(\cdot, \cdot)$ on $\Sset$.
\end{proposition}

\begin{proof}
Let $K = p(f)$. We must show that for any $x, z \in \Sset$, $|f(x) - f(z)| \le K \cdot d_H(x, z)$.

\textbf{Case 1: $d_H(x, z) = 1$.}
If $d_H(x, z) = 1$, then $x$ and $z$ differ in exactly one coordinate, say $j$. This means $x_{-j} = z_{-j}$. By the definition of $\delta_j(f)$:
\[
|f(x) - f(z)| \le \sup_{x'_{-j}=y'_{-j}} |f(x') - f(y')| = \delta_j(f).
\]
By the definition of $p(f)$, $\delta_j(f) \le \max_i \delta_i(f) = p(f) = K$.
Therefore, $|f(x) - f(z)| \le K = K \cdot d_H(x, z)$. This also shows that $K$ is the best Lipschitz constant for pairs at distance 1.

\textbf{Case 2: $d_H(x, z) = k > 1$.}
Let the coordinates where $x$ and $z$ differ be $I = \{i_1, \dots, i_k\}$. We can construct a path from $x$ to $z$ by changing one coordinate at a time. Let $z^{(0)} = x$, and let $z^{(t)}$ be the state obtained by changing the first $t$ coordinates in $I$ from their values in $x$ to their values in $z$.
For example, $z^{(1)}$ is identical to $x$ except $z^{(1)}_{i_1} = z_{i_1}$.
In general, $z^{(t)}$ and $z^{(t+1)}$ differ only in coordinate $i_{t+1}$. Thus, $d_H(z^{(t)}, z^{(t+1)}) = 1$.
The full path is $x = z^{(0)}, z^{(1)}, \dots, z^{(k)} = z$.

Using the triangle inequality, we ``telescope" the sum:
\begin{align*}
|f(x) - f(z)| &= |f(z^{(0)}) - f(z^{(k)})| \\
&= \left| \sum_{t=0}^{k-1} \big( f(z^{(t)}) - f(z^{(t+1)}) \big) \right| \\
&\le \sum_{t=0}^{k-1} |f(z^{(t)}) - f(z^{(t+1)})|.
\end{align*}
For each term in the sum, $z^{(t)}$ and $z^{(t+1)}$ differ in exactly one coordinate, so $d_H(z^{(t)}, z^{(t+1)}) = 1$. From Case 1, we know:
\[
|f(z^{(t)}) - f(z^{(t+1)})| \le p(f) = K.
\]
Substituting this into the sum:
\[
|f(x) - f(z)| \le \sum_{t=0}^{k-1} K = k \cdot K.
\]
Since $k = d_H(x, z)$, we have shown:
\[
|f(x) - f(z)| \le K \cdot d_H(x, z).
\]
This holds for all $x, z \in \Sset$, so $p(f)$ is the Lipschitz constant of $f$ w.r.t. $d_H$.

\paragraph{Optimality.}
Let L be any constant such that $|f(x)-f(y)| \le L$, $d_H(x,y)$ for all $x,y$.
Fix $i$ and $x,y$ with $x_{-i}=y_{-i}$; then $d_H(x,y)=1$, so $|f(x)-f(y)| \le L$.
Taking the supremum over such pairs gives $\delta_i(f) \le L$, hence
$p(f)=\max_i \delta_i(f) \le L$. Therefore $p(f)$ is the least Lipschitz constant.
\end{proof}

\end{document}